\documentclass[a4paper,10pt,oneside,onecolumn,review,number,sort&compress]{elsarticlefake}

\usepackage{amssymb}
\usepackage{multirow}
\usepackage{color}
\usepackage{graphicx}
\usepackage{amsthm}
\usepackage{amssymb}
\usepackage{epstopdf}
\usepackage{xcolor,colortbl}
\usepackage{hyperref}
\usepackage{footnote}
\usepackage{pbox}
\makesavenoteenv{tabular}
\usepackage{listings}
\usepackage{epsfig,psfrag}
\usepackage[T1]{fontenc}
\usepackage{color}
\usepackage{amsmath}
\usepackage{amsbsy}
\usepackage{lipsum} 
\usepackage[sc]{mathpazo} 
\usepackage[T1]{fontenc} 
\usepackage{microtype} 

\usepackage{multicol} 
\usepackage[hang, small,labelfont=bf,up,textfont=it,up]{caption} 
\usepackage{booktabs} 
\usepackage{float} 
\usepackage{hyperref} 

\usepackage{lettrine} 
\usepackage{paralist}
\usepackage[linesnumbered,lined,boxed,commentsnumbered]{algorithm2e}
\SetKwRepeat{Do}{do}{while}
\newtheorem{mydef}{Definition}
\newtheorem{mylemma}{Lemma}
\newtheorem{mytheorem}{Theorem}

\journal{Science of Computer Programming}

\begin{document}

\begin{frontmatter}



\title{Counterexample Guided Inductive Optimization}

 \author[label1]{Rodrigo F. Ara\'ujo}
 \author[label2]{Higo F. Albuquerque}
 \author[label2]{Iury V. de Bessa}
 \author[label3]{Lucas C. Cordeiro}
 \author[label2]{Jo{\~a}o Edgar C. Filho}
\address[label1]{Federal Institute of Amazonas, Campus Manaus Distrito Industrial, Avenida Governador Danilo Areosa, 1672, Distrito Industrial, 69075351, Manaus, AM, Brasil}
\address[label2]{Department of Electricity, Federal University of Amazonas, Avenida General Rodrigo Oct{\'a}vio, 6200, Coroado I, Manaus, AM, 69077-000}
\address[label3]{Department of Computer Science, University of Oxford, Wolfson Building, Parks Road, OXFORD, OX1 3QD, UK}


\address{}

\begin{abstract}

This paper describes three variants of a counterexample guided inductive optimization (CEGIO) approach based on Satisfiability Modulo Theories (SMT) solvers. In particular, CEGIO relies on iterative executions to constrain a verification procedure, in order to perform inductive generalization, based on counterexamples extracted from SMT solvers. CEGIO is able to successfully optimize a wide range of functions, including non-linear and non-convex optimization problems based on SMT solvers, in which data provided by counterexamples are employed to guide the verification engine, thus reducing the optimization domain. The present algorithms are evaluated using a large set of benchmarks typically employed for evaluating optimization techniques. Experimental results show the efficiency and effectiveness of the proposed algorithms, which find the optimal solution in all evaluated benchmarks, while traditional techniques are usually trapped by local minima.

\end{abstract}

\begin{keyword}


Satisfiability Modulo Theories  (SMT) \sep Model checking \sep Global optimization \sep Non-convex optimization
\end{keyword}

\end{frontmatter}



\section{Introduction}
\label{sec:intro}

Optimization is an important research topic in many fields, especially in computer science and engineering~\cite{deb2004optimization}. Commonly, scientists and engineers have to find parameters, which optimize the behavior of a given system or the value of a given function ({\it i.e.}, an optimal solution). Optimization characterizes and distinguishes the engineering gaze over a problem; for this particular reason, previous studies showed that optimization is one of the main differences between engineering design and technological design~\cite{Gattie2007}.

Computer science and optimization maintain a symbiotic relationship. Many important advances of computer science are based on optimization theory. As example, planning and decidability problems ({\it e.g.}, game theory~\cite{Shoham2008}), resource allocation problems ({\it e.g.}, hardware/software co-design~\cite{Teich2012}), and computational estimation and approximation ({\it e.g.}, numerical analysis~\cite{kowalski1995selected}) represent important optimization applications. Conversely, computer science plays an important role in recent optimization studies, developing efficient algorithms and providing respective tools for supporting model management and results analysis~\cite{derigs2009optimization}.

There are many optimization techniques described in the literature ({\it e.g.}, simplex~\cite{garfinkel1972integer}, gradient descent~\cite{Bartholomew--Biggs2008}, and genetic algorithms~\cite{goldberg1989genetic}), which are suitable for different classes of optimization problems ({\it e.g.}, linear or non-linear, continuous or discrete, convex or non-convex, and single- or multi-objective). These techniques are usually split into two main groups: deterministic and stochastic optimization. Deterministic optimization is the classic approach for optimization algorithms, which is based on calculus and algebra operators, {\it e.g.}, gradients and Hessians~\cite{cavazzuti2012optimization}.  Stochastic optimization employs randomness in the optima search procedure~\cite{cavazzuti2012optimization}. This paper presents a novel class of search-based optimization algorithm that employs non-deterministic representation of decision variables and constrains the state-space search based on counterexamples produced by an SMT solver, in order to ensure the complete global optimization without employing randomness. This class of techniques is defined here as \textit{counterexample guided inductive optimization} (CEGIO), which is inspired by the syntax-guided synthesis (SyGuS) to perform inductive generalization based on counterexamples provided by a verification oracle~\cite{AlurSyGus2013}.


Particularly, a continuous non-convex optimization problem is one of the most complex problems. As a result, several traditional methods ({\it e.g.}, Newton-Raphson~\cite{deb2004optimization} and Gradient Descent~\cite{Bartholomew--Biggs2008}) are inefficient to solve that specific class of problems~\cite{deb2004optimization}. Various heuristics are developed for obtaining approximated solutions to those problems; heuristics methods ({\it e.g.}, ant colony~\cite{Dorigo2006} and genetic algorithms~\cite{goldberg1989genetic}) offer faster solutions for complex problems, but they sacrifice the system's correctness and are easily trapped by local optimal solutions.

This paper presents a novel counterexample guided inductive optimization technique based on SMT solvers, which is suitable for a wide variety of functions, even for non-linear and non-convex functions, since most real-world optimization problems are non-convex. The function evaluation and the search for the optimal solution is performed by means of an iterative execution of successive verifications based on counterexamples extracted from SMT solvers. The counterexample provides new domain boundaries and new optima candidates. In contrast to other heuristic methods ({\it e.g.}, genetic algorithms), which are usually employed for optimizing this class of function, the present approaches always find the global optimal point.

This study extends the previous work of Ara\'{u}jo {\it et al.}~\cite{Araujo2016} and presents three variants of a counterexample guided inductive optimization approach based on SMT solvers, which improve the technique performance for specific class of functions. Furthermore, the experimental evaluation is largely expanded, since the algorithms are executed for additional optimization problems and the performance of each proposed algorithm is compared to six well-known optimization techniques. The present CEGIO approaches are able to find the correct global minima for 100\% of the benchmarks, while other techniques are usually trapped by local minima, thus leading to incorrect solutions.

\subsection{Contributions}

Our main original contributions are:

\begin{itemize}
\item{\textbf{Novel counterexample guided inductive optimization approach}. This work describes three novel variants of a counterexample guided inductive optimization approach based on SMT solvers: generalized, simplified, and fast algorithms. The generalized algorithm can be used for any constrained optimization problem and presents minor improvements w.r.t. Ara\'{u}jo {\it et al.}~\cite{Araujo2016}. The simplified algorithm is faster than the generalized one and can be employed if information about the minima location is provided, {\it e.g.}, the cost function is semi-definite positive. The fast algorithm presents a  significant speed-up if compared to the generalized and simplified ones,  but it can only be employed for convex functions.}

\item{\textbf{Convergence Proofs}. This paper presents proofs of convergence and completeness (omitted in Ara\'{u}jo {\it et al.}~\cite{Araujo2016}) for the proposed counterexample guided inductive optimization algorithms.}

\item{\textbf{SMT solvers performance comparison}. The experiments are performed with three different SMT solvers: Z3~\cite{z3}, Boolector~\cite{boolector}, and MathSAT~\cite{mathsat5}. The experimental results show that the solver choice can heavily influence the method performance.}

\item{\textbf{Additional benchmarks}. The benchmark suite is expanded to $30$ optimization functions extracted from the literature~\cite{functionslist}.}

\item{\textbf{Comparison with existing techniques.} The proposed technique is compared to genetic algorithm~\cite{goldberg1989genetic}, particle swarm~\cite{pswarmref:olsson2011particle}, pattern search~\cite{patternsref:alberto2004pattern}, simulated annealing~\cite{saref:Laarhoven:1987:SAT:59580}, and nonlinear programming~\cite{npref:byrd2000trust}, which are traditional optimization techniques employed for non-convex functions.}
\end{itemize}

\subsection{Availability of Data and Tools}

Our experiments are based on a set of publicly available benchmarks. All tools, benchmarks, and results of our evaluation are available on a supplementary web page \texttt{http://esbmc.org/benchmarks/jscp2017.zip}.

\subsection{Outline}

Section~\ref{sec:related} discusses related studies. Section~\ref{sec:preliminaries} provides an overview of optimization problems and techniques, and describes background on software model checking. Section~\ref{sec:smtoptmodel} describes the ANSI-C model developed for optimization problems that is suitable for the counterexample guided inductive optimization procedure. Section~\ref{sec:smtoptnonconvex} describes the generalized and simplified optimization algorithms and respective completeness proofs. Section~\ref{sec:smtoptconvex} describes the fast optimization algorithm and respective completeness proof. Section~\ref{sec:experiments} reports the experimental results for evaluating all proposed optimization algorithms, while Section~\ref{sec:conc} concludes this work and proposes further studies.

\section{Related Work}
\label{sec:related}

SMT solvers have been widely applied to solve several types of verification, synthesis, and optimization problems. They are typically used to check the satisfiability of a logical formula, returning assignments to variables that evaluate the formula to true, if it is satisfiable; otherwise, the formula is said to be unsatisfiable. Nieuwenhuis and Oliveras~\cite{Nieuwenhuis2006} presented the first research about the application of SMT to solve optimization problems. Since then, SMT solvers have been used to solve different optimization problems, {\it e.g.}, minimize errors in linear fixed-point arithmetic computations in embedded control software~\cite{Eldib2014}; reduce the number of gates in FPGA digital circuits~\cite{Estrada2003}; hardware/software partition in embedded systems to decide the most efficient system implementation~\cite{TrindadeSBESC2015, TrindadeSBESC2013, TrindadeDAES2016}; and schedule applications for a multi-processor platform~\cite{Cotton2011}. All those previous studies use SMT-based optimization over a Boolean domain to find the best system configuration given a set of metrics. In particular, in Cotton {\it et al.}~\cite{Cotton2011} the problem is formulated as a multi-objective optimization problem. Recently, Shoukry \textit{et al.}~\cite{CDC2016} proposed a scalable solution for synthesizing a digital controller and motion planning for under-actuated robots from LTL specifications. Such solution is more flexible and allows solving a wider variety of problems, but they are focused on optimization problems that can be split into a Boolean part and other convex part. 

In addition, there were advances in the development of different specialized SMT solvers that employ generic optimization techniques to accelerate SMT solving, {\it e.g.}, the \texttt{ABsolver}~\cite{ABsolver}, which is used for automatic analysis and verification of hybrid-system and control-system. The \texttt{ABsolver} uses a non-linear optimization tool for Boolean and polynomial arithmetic and a lazy SMT procedure to perform a faster satisfiability checking. Similarly, \texttt{CalCs}~\cite{Nuzzo2010} is also an SMT solver that combines convex optimization and lazy SMT to determine the satisfiability of conjunctions of convex non-linear constraints. Recently, Shoukry \textit{et al.}~\cite{SMC} show that a particular class of logic formulas (named SMC formulas) generalizes a wide range of formulas over Boolean and nonlinear real arithmetic, and propose the Satisfiability Modulo Convex Optimization to solve satisfiability problems over SMC formulas. Our work differs from those previous studies~\cite{ABsolver,Nuzzo2010,SMC} since it does not focus on speeding up SMT solvers, but it employs an SMT-based model-checking tool to guide (via counterexample) an optimization search procedure in order to ensure the global optimization.

Recently, $\nu Z$~\cite{Bjorner2015} extends the SMT solver Z3 for linear optimization problems; Li \textit{et al.} proposed the \texttt{SYMBA} algorithm~\cite{Symba}, which is an SMT-based symbolic optimization algorithm that uses the theory of linear real arithmetic and SMT solver as black box. Sebastiani and Trentin~\cite{SebastianiCAV2015} present \texttt{OptiMathSat}, which is an optimization tool that extends MathSAT5 SMT solver to allow solving linear functions in the Boolean, rational, and integer domains or a combination of them; in Sebastiani and Tomasi~\cite{SebastianiTransactions2015}, the authors used a combination of SMT and LP techniques to minimize rational functions; the related work~\cite{SebastianiTACAS2015} extends their work with linear arithmetic on the mixed integer/rational domain, thus combining SMT, LP, and ILP techniques. 

As an application example, Pavlinovic {\it et al.} \cite{Pavlinovic0W15} propose an approach which considers all possible compiler error sources for statically typed functional programming languages and reports the most useful one subject to some usefulness criterion. The authors formulate this approach as an optimization problem related to SMT and use $\nu$Z to compute an optimal error source in a given ill-typed program. The approach described by Pavlinovic {\it et al.}, which uses MaxSMT solver $\nu$Z, shows a significant performance improvement if compared to previous SMT encodings and localization algorithms.

Most previous studies related to SMT-based optimization can only solve linear problems over integer, rational, and Boolean domains in specific cases, leading to limitations in practical engineering applications. Only a few studies~\cite{CDC2016} are able to solve non-linear problems, but they are also constrained to convex functions. In contrast, this paper proposes a novel counterexample guided inductive optimization method based on SMT solvers to minimize functions, linear or non-linear, convex or non-convex, continuous or discontinuous. As a result, the proposed methods are able to solve optimization problems directly on the rational domain with adjustable precision, without using any other technique to assist the state-space search. Furthermore, our proposed methods employ a model-checking tool to generate automatically SMT formulas from an ANSI-C model of the optimization problem, which makes the representation of problems for SMT solving easy for engineers.

\section{Preliminaries}
\label{sec:preliminaries}

\subsection{Optimization Problems Overview}
\label{ssec:classification}

Let $f:X\rightarrow\mathbb{R}$ be a cost function, such that $X\subset \mathbb{R}^n$ represents the decision variables vector $x_{1},x_{2},...,x_{n}$ and $f(x_{1},x_{2},...,x_{n})\equiv f(\textbf{x})$. Let $\Omega\subset X$ be a subset settled by a set of constraints. 

\begin{mydef}
\label{def:multivariable}
A multi-variable optimization problem consists in finding an optimal vector $\textbf{x}$, which minimizes $f$ in $\Omega$. 
\end{mydef}

According to Definition~\ref{def:multivariable}, an optimization problem can be written as
\begin{equation}
\label{eq:optproblem}
\begin{array}{cc}
\min & f(\textbf{x}),  \\
\textrm{ s.t. } & \textbf{x}\in\Omega.\\
\end{array}
\end{equation}

In particular, this optimization problem can be classified in different ways w.r.t. constraints, decision variables domain, and nature of cost function $f$. All optimization problems considered here are constrained, {\it i.e.}, decision variables are constrained by the subset $\Omega$. The optimization problem domain $X$ that contains $\Omega$ can be the set of $\mathbb{N}$, $\mathbb{Z}$, $\mathbb{Q}$, or $\mathbb{R}$. Depending on the domain and constraints, the optimization search-space can be small or large, which influences the optimization algorithms performance. 

The cost function can be classified as linear or non-linear; continuous, discontinuous or discrete; convex or non-convex. Depending on the cost function nature, the optimization problem can be hard to solve, given the time and memory resources~\cite{Galperin19911}. Particularly, non-convex optimization problems are the most difficult ones w.r.t. the cost function nature. A non-convex cost function is a function whose epigraph is a non-convex set and consequently presents various inflexion points that can trap the optimization algorithm to a sub-optimal solution. A non-convex problem is necessarily a non-linear problem and it can also be discontinuous. Depending on that classification, some optimization techniques are unable to solve the optimization problem, and some algorithms usually point to suboptimal solutions, {\it i.e.}, a solution that is not a global minimum of $f$, but it only locally minimizes $f$. Global optimal solutions of the function $f$, aforementioned, can be defined as
\begin{mydef}
\label{def:global}
A vector $\textbf{x}^{*} \in \Omega$ is a global optimal solution of $f$ in $\Omega$ {\it iff} $f(\textbf{x}^{*})\leq f(x), \forall \textbf{x} \in \Omega$.
\end{mydef}
%

\subsection{Optimization Techniques}
\label{ssec:techniques}

Different optimization problems offer different difficulties to their particular solutions. Such complexity is mainly related to the ruggedness ({\it e.g.}, continuity, differentiability, and smoothness) and dimensionality of the problem ({\it i.e.}, the dimension, and for the finite case, the number of elements of $\Omega$). Depending on these factors, different optimization techniques can be more efficient to solve a particular optimization problem. Generally, traditional optimization techniques can be divided into two groups: deterministic and stochastic optimization. 

The deterministic techniques employ a search engine, where each step is directly and deterministically related to the previous steps~\cite{refdet:floudas2000deterministic}. In summary, deterministic techniques can be gradient-based or enumerative search-based. Gradient-based techniques search for points, where the gradient of cost function is null ($\nabla f(\textbf{x})= 0$), {\it e.g.}, gradient-descent~\cite{snyman2005practical} and Newton's optimization~\cite{deb2004optimization}. Although they are fast and efficient, those techniques are unable to solve non-convex or non-differentiable problems. Enumerative search-based optimization consists in scanning the search-space by enumerating all possible points and comparing cost function with best previous values, {\it e.g.}, dynamic programming, branch and bound~\cite{scholz2011deterministic}, and pattern search~\cite{FINDLER198741}.

Stochastic techniques employ randomness to avoid the local minima and to ensure the global optimization; such techniques are usually based on meta-heuristics~\cite{refstc:marti2005stochastic}. This class of techniques has become very popular in the last decades and has been used in all types of optimization problems. Among those stochastic techniques, simulated annealing~\cite{saref:Laarhoven:1987:SAT:59580}, particle swarm~\cite{pswarmref:olsson2011particle}, and evolutionary algorithms ({\it e.g.}, genetic algorithms~\cite{goldberg1989genetic}) are usually employed in practice. 

Recently, optimization techniques and tools that employ SMT solvers and non-deterministic variables were applied to solve optimization problems~\cite{CDC2016,ABsolver,Nuzzo2010,Gao2013,Bjorner2015,Symba,SebastianiCAV2015,SebastianiTransactions2015,SebastianiTACAS2015}, which searches for the global optima in a search-space that is symbolically defined and uses counterexamples produced by SMT solvers to further constrain the search-space. The global optima is the set of values for the decision variables that makes an optimization proposition satisfiable. The technique presented here is the first optimization method based on SMT solvers and inductive generalization described in the literature, which is able to solve non-convex problems over $\mathbb{R}$.

\subsection{Model Checking}
\label{ssec:mc}

Model checking is an automated verification procedure to exhaustively check all (reachable) system's states~\cite{Clarke1999}. 
The model checking procedure typically consists of three steps: modeling, specification, and verification.

Modeling is the first step, where it converts the system to a formalism that is accepted by a verifier. The modeling step usually requires the use of an abstraction to eliminate irrelevant (or less) important system details~\cite{Baier2008}. The second step is the specification, which describes the system's behavior and the property to be checked. An important issue in the specification is the correctness. Model checking provides ways to check whether a given specification satisfies a system's property, but it is difficult to determine whether such specification covers all properties in which the system should satisfy. 

Finally, the verification step checks whether a given property is satisfied w.r.t. a given model, {\it i.e.}, all relevant system states are checked to search for any state that violates the verified property. In case of a property violation, the verifier reports the system's execution trace (counterexample), which contains all steps from the (initial) state to the (bad) state that lead to the property violation. Errors could also occur due to incorrect system modeling or inadequate specification, thus generating false verification results.

\subsubsection{Bounded Model Checking (BMC)}
\label{ssec:smt-mc}



BMC is an important verification technique, which has presented attractive results over the last years~\cite{Beyer2016}. BMC techniques based on Boolean Satisfiability (SAT)~\cite{handbook09} or Satisfiability Modulo Theories (SMT)~\cite{BarrettSST09} have been successfully applied to verify single- and multi-threaded programs, and also to find subtle bugs in real programs~\cite{esbmc,CBMC}. BMC checks the negation of a given property at a given depth over a transition system $M$. 

\begin{mydef}{\cite{handbook09} -- }
Given a transition system $M$, a property $\phi$, and a bound $k$; BMC unrolls the system $k$ times and translates it into a verification condition (VC) $\psi$, which is satisfiable iff $\phi$ has a counterexample of depth less than or equal to $k$. 
\end{mydef}

In this study, the ESBMC tool~\cite{MorseRCN014} is used as verification engine, as it represents one of the most efficient BMC tools that participated in the last software verification competitions~\cite{Beyer2016}. 
%
ESBMC finds property violations such as pointer safety, array bounds, atomicity, overflows, deadlocks, data race, and memory leaks in single- and multi-threaded C/C++ software. It also verifies programs that make use of bit-level, pointers, structures, unions, fixed- and floating-point arithmetic. Inside ESBMC, the associated problem is formulated by constructing the following logical formula
\begin{equation}
	\label{esbmc_logical_formula}
	\psi_k = I(S_{0}) \wedge \bigvee_{i=0}^{k} \bigwedge_{j=0}^{i-1} \gamma (s_{j},s_{j+1})  \wedge \overline{\phi(s_{1})}
\end{equation}
where $\phi$ is a property and $S_{0}$ is a set of initial states of $M$, and $\gamma(s_{j},s_{j+1})$ is the transition relation of $M$ between time steps $j$ and $j+1$. Hence, $I(S_{0}) \wedge \bigwedge_{j=0}^{i-1} \gamma (s_{j},s_{j+1}) $ represents the executions of a transition system $M$ of length $i$. The above VC $\psi$ can be satisfied if and only if, for some $i \leq k$ there exists a reachable state at time step $i$ in which $\phi$ is violated.
If the logical formula (\ref{esbmc_logical_formula}) is satisfiable ({\it i.e.}, returns $true$), then the SMT solver provides a satisfying assignment (counterexample).%
\begin{mydef}
\label{counterexample}
A counterexample for a property $\phi$ is a sequence of states $s_{0}, s_{1},\ldots, s_{k}$ with $s_{0} \in S_{0}$, $s_{k} \in S_k$, and $\gamma\left(s_{i}, s_{i+1}\right)$ for $0 \leq i < k$ that makes (\ref{esbmc_logical_formula}) satisfiable. If it is unsatisfiable (i.e., returns \textit{false}), then one can conclude that there is no error state in $k$ steps or less.
\end{mydef}
In addition to software verification, ESBMC has been applied to ensure correctness of digital filters and controllers~\cite{Abreu2016,Bessa2016,esbmc_controller_SBESC}. Recently, ESBMC has been applied to optimize HW/SW co-design~\cite{TrindadeSBESC2015,TrindadeSBESC2013,TrindadeDAES2016}.

\section{Verification Model for Counterexample Guided Inductive Optimization}
\label{sec:smtoptmodel}


\subsection{Modeling Optimization Problems using a Software Model Checker}
\label{ssec:model}

There are two important directives in the C/C++ programming language, which can be used for modeling and controlling a verification process: \texttt{ASSUME} and \texttt{ASSERT}. The \texttt{ASSUME} directive can define constraints over (non-deterministic) variables, and the \texttt{ASSERT} directive is used to check system's correctness w.r.t. a given property. Using these two statements, any off-the-shelf C/C++ model checker ({\it e.g.}, CBMC~\cite{CBMC}, CPAChecker~\cite{Beyer2011}, and ESBMC~\cite{MorseRCN014}) can be applied to check specific constraints in optimization problems, as described by Eq.~\eqref{eq:optproblem}. 

Here, the verification process is iteratively repeated to solve an optimization problem using intrinsic functions available in ESBMC ({\it e.g.}, \textit{\_\_ESBMC\_assume} and \textit{\_\_ESBMC\_assert}). We apply incremental BMC to efficiently prune the state-space search based on counterexamples produced by an SMT solver. Note that completeness is not an issue here (cf. Definitions~\ref{def:multivariable} and~\ref{def:global}) since our optimization problems are represented by loop-free programs~\cite{Gadelha2015}.

\subsection{Illustrative Example}
\label{ssec:example}

The Ursem03's function is employed to illustrate the present SMT-based optimization method for non-convex optimization problems~\cite{functionslist}. The Ursem03's function is represented by a two-variables function with only one global minimum in $f(x_{1},x_{2})=-3$, and has four regularly spaced local minima positioned in a circumference, with the global minimum in the center. Ursem03's function is defined by Eq.~\eqref{eq:ursem}; Fig.~\ref{fig:ursem} shows its respective graphic.

\begin{equation}
\footnotesize
\label{eq:ursem}
f(x_{1},x_{2})= -\sin\left(2.2\pi x_{1}-\frac{\pi}{2}\right) \frac{(2-|x_1|)(3-|x_1|)}{4} -\sin\left(2.2\pi x_{2}-\frac{\pi}{2}\right) \frac{(2-|x_2|)(3-|x_2|)}{4}
\end{equation}
\begin{figure}[htb]
\centering
\includegraphics[width=0.9\columnwidth]{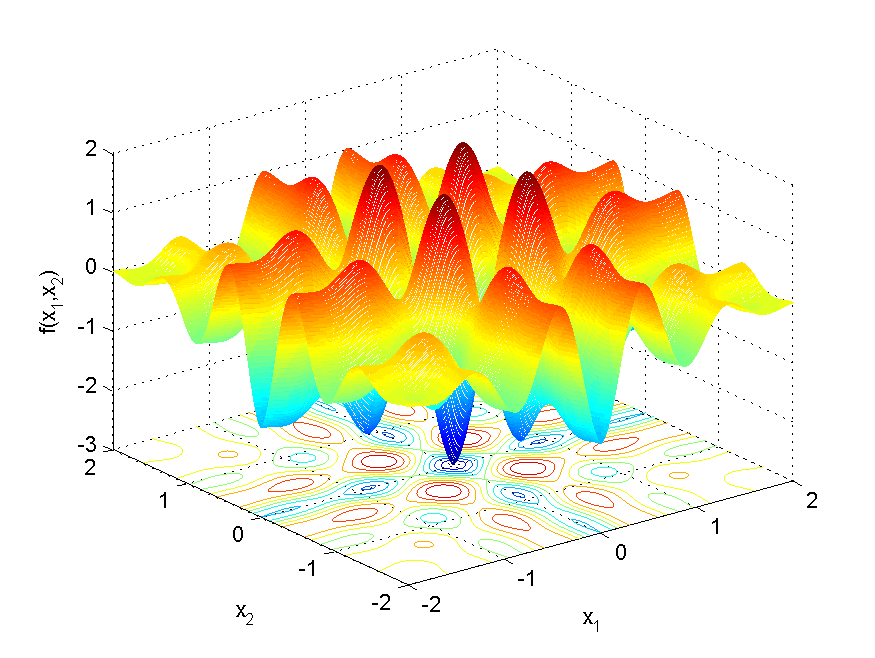}
\caption{Ursem03's function \label{fig:ursem}}
\end{figure}
%

\subsection{Modeling}

The modeling process defines constraints, {\it i.e.}, $\Omega$ boundaries (cf.~Section~\ref{ssec:classification}). This step is important for reducing the state-space search and consequently for avoiding the state-space explosion by the underlying model-checking procedure. Our verification engine is not efficient for unconstrained optimization; fortunately, the verification time can be drastically reduced by means of a suitable constraint choice. Consider the optimization problem given by Eq.~\eqref{eq:ursemoptproblem1}, which is related to the Ursem03's function given in Eq.~\eqref{eq:ursem}:
\begin{equation}
\label{eq:ursemoptproblem1}
\begin{array}{cc}
\min & f(x_{1},x_{2}) \\
\textrm{ s.t. } & x_{1}\geq 0 \\
 & x_{2}\geq 0.
\end{array}
\end{equation}

Note that inequalities $x_{1}\geq 0$ and $x_{2}\geq 0$ are pruning the state-space search to the first quadrant; however, even so it produces a (huge) state-space to be explored since $x_{1}$ and $x_{2}$ can assume values with very high modules. The optimization problem given by Eq.~\eqref{eq:ursemoptproblem1} can be properly rewritten as Eq.~\eqref{eq:ursemoptproblem2} by introducing new constraints. The boundaries are chosen based on Jamil and Yang~\cite{functionslist}, which define the domain in which the optimization algorithms can evaluate the benchmark functions.
\begin{equation}
\label{eq:ursemoptproblem2}
\begin{array}{cc}
\min & f(x_{1},x_{2})  \\
\textrm{ s.t. } & -2\leq x_{1}\leq 2 \\
 &-2 \leq x_{2} \leq 2.
\end{array}
\end{equation}

From the optimization problem definition given by Eq.~\eqref{eq:ursemoptproblem2}, the modeling step can be encoded, where decision variables are declared as non-deterministic variables constrained by the \texttt{ASSUME} directive. In this case,  $-2\leq x_{1}\leq 2$ and $-2 \leq x_{2} \leq 2$.  Fig.~\ref{alg:model} shows the respective C code for modeling Eq.~\eqref{eq:ursemoptproblem2}.
\begin{figure}[ht]
\small
\begin{lstlisting}[xleftmargin=.025\textwidth,xrightmargin=.025\textwidth,frame=single,]
#include "math2.h"
float nondet_float();
int main() {
  //define decision variables
  float x1 = nondet_float();
  float x2 = nondet_float();	
  //constrain the state-space search
  __ESBMC_assume((x1>=-3) && (x1<=3));
  __ESBMC_assume((x2>=-2) && (x2<=2));
  //computing Ursem's function 
  float fobj;
  fobj= -sin(2.2*pi*x1-pi/2)*(2-abs(x1))(3-abs(x1))/4 
  -sin(2.2*pi*x2-pi/2)*(2-abs(x2))(3-abs(x2))/4;
  return 0;
}
\end{lstlisting}
\caption{C Code for the optimization problem given by Eq.~\eqref{eq:ursemoptproblem2}.}
\label{alg:model}
\end{figure}
Note that in Figure~\ref{alg:model}, the decision variables $x_1$ and $x_2$ are declared as floating-point numbers initialized with non-deterministic values; we then constraint the state-space search using assume statements. The objective function of Ursem`s function is then declared as described by Eq.~\ref{eq:ursem}. 

\subsection{Specification}

The next step of the proposed methodology is the specification, where the system behavior and the property to be checked are described. For the Ursem03's function, the result of the specification step is the C program shown in Fig.~\ref{alg:specify}, which is iteratively checked by the underlying verifier. Note that the decision variables are declared as integer type and their initialization depends on a given precision $p$, which is iteratively adjusted once the counterexample is produced by the SMT solver. Indeed, the C program shown in Fig.~\ref{alg:model} leads the verifier to produce a considerably large state-space exploration, if the decision variables are declared as non-deterministic floating-point type. In this study, decision variables are defined as non-deterministic integers, thus discretizing and reducing the state-space exploration; however, this also reduces the optimization process precision.

To trade-off both precision and verification time, and also to maintain convergence to an optimal solution, the underlying model-checking procedure has to be iteratively invoked, in order to increase its precision for each successive execution. An integer variable $p=10^n$ is created and iteratively adjusted, such that $n$ is the amount of decimal places related to the decision variables. Additionally, a new constraint is inserted; in particular, the new value of the objective function $f(\textbf{x}^{(i)})$ at the $i$-th must not be greater than the value obtained in the previous iteration $f(\textbf{x}^{(i-1)})$. Initially, all elements in the state-space search $\Omega$ are candidates for optimal points, and this constraint cutoffs several candidates on each iteration.

In addition, a property has to be specified to ensure convergence to the minimum point on each iteration. This property specification is stated by means of an assertion, which checks whether the literal $l_{optimal}$ given in Eq.~\eqref{eq:optimal_literal} is satisfiable for every optimal candidate $f_{c}$ remaining in the state-space search ({\it i.e.}, traversed from lowest to highest).
\begin{equation}
\label{eq:optimal_literal}
l_{optimal}\iff f(\textbf{x})>f_{c}
\end{equation}

The verification procedure stops when the literal $l_{optimal}$ is not satisfiable, {\it i.e.}, if there is any $\textbf{x}^{(i)}$ for which $f(\textbf{x}^{(i)})\leq f_{c}$; a counterexample shows such $\textbf{x}^{i}$, approaching iteratively $f(\textbf{x})$ from the optimal $\textbf{x}^{*}$. Fig.~\ref{alg:specify} shows the initial specification for the optimization problem given by Eq.~\eqref{eq:ursemoptproblem2}. The initial value of the objective function can be randomly initialized. For the example in Fig.~\ref{alg:specify}, $f(\textbf{x}^{(0)})$ is arbitrarily initialized to $100$, but the present optimization algorithm works for any initial state.

%
\begin{figure}[ht]
\small
\begin{lstlisting}[xleftmargin=.025\textwidth,xrightmargin=.025\textwidth,frame=single,]
#include "math2.h"
#define p 1 //precision variable
int nondet_int();
float nondet_float();
int main() {
  float f_i = 100;	//previous objective function value
  int lim_inf_x1 = -2*p;
  int lim_sup_x1 = 2*p;
  int lim_inf_x2 = -2*p;
  int lim_sup_x2 = 2*p;
  int X1 = nondet_int();
  int X2 = nondet_int();
  float x1 = float nondet_float();
  float x2 = float nondet_float();
  __ESBMC_assume( (X1>=lim_inf_x1) && (X1<=lim_sup_x1) );
  __ESBMC_assume( (X2>=lim_inf_x2) && (X2<=lim_sup_x2) );
  __ESBMC_assume( x1 = (float) X1/p );
  __ESBMC_assume( x2 = (float) X2/p );
  float fobj;
  fobj= -sin(2.2*pi*x1-pi/2)*(2-abs(x1))(3-abs(x1))/4 
  -sin(2.2*pi*x2-pi/2)*(2-abs(x2))(3-abs(x2))/4;
  //constrain to exclude fobj>f_i
  __ESBMC_assume( fobj < f_i );	
  assert( fobj > f_i );
  return 0;
}
\end{lstlisting}
\caption{C code after the specification of Eq.~\eqref{eq:ursemoptproblem2}.}
\label{alg:specify}
\end{figure}

\subsection{Verification}

Finally, in the verification step, the C program shown in Fig.~\ref{alg:specify} is checked by the verifier and a counterexample is returned with a set of decision variables $\textbf{x}$, for which the objective function value converges to the optimal value. A specified C program only returns a successful verification result if the previous function value is the optimal point for that specific precision (defined by $p$), {\it i.e.}, $f(\textbf{x}^{(i-1)})=f(\textbf{x}^{*})$. For the example shown in Fig.~\ref{alg:specify}, the verifier shows a counterexample with the following decision variables: $x_{1}=2$ and $x_{2}=0$. These decision variable are used to compute a new minimum candidate, note that $f(2,0)=-1.5$, which is the new minimum candidate solution provided by this verification step. Naturally, it is less than the initial value (100), and this verification can be repeated with the new value of $f(\textbf{x}^{(i-1)})$, in order to obtain an objective function value that is close to the optimal point on each iteration. Note that the data provided by the counterexample is crucial for the algorithm convergence and for the state-space search reduction.



\section{Counterexample Guided Inductive Optimization of Non-convex Functions}
\label{sec:smtoptnonconvex}

This section presents two variants of the Counterexample Guided Inductive Optimization 
(CEGIO) algorithm for global constrained optimization. A generalized CEGIO algorithm 
is explained in Subsection~\ref{ssec:global}, together with
a convergence proof in Subsection~\ref{ssec:proof}, while Subsection~\ref{ssec:simplified} 
presents a simplified version of that algorithm. 

\subsection{CEGIO: the Generalized Algorithm (CEGIO-G)}
\label{ssec:global}

The generalized SMT-based optimization algorithm previously presented by Ara\'ujo {\it et al.}~\cite{Araujo2016} 
is able to find the global optima for any optimization problem that can be modeled with the methodology presented 
in Section~\ref{sec:smtoptmodel}. The execution time of that algorithm depends on how the state-space search is 
restricted and on the number of the solution decimal places. Specifically, the algorithm presents a fixed-point solution with 
adjustable precision, \textit{i.e.}, the number of decimal places can be defined. Naturally, for integer optimal points, 
this algorithm returns the correct solution quickly. However, this algorithm might take longer for achieving the optimal 
solution of unconstrained optimization problems with non-integer solutions since it depends on the required precision. 
Although this algorithm frequently produces a longer execution time than other traditional techniques, 
its error rate is typically lower than other existing methods, once it is based on a complete and sound verification procedure. 
Alg.~\ref{alg:optalg1} shows an improved version of the algorithm presented by Ara\'ujo {\it et al.}~\cite{Araujo2016}; 
this algorithm is denoted here as ''Generalized CEGIO algorithm'' (CEGIO-G).
\begin{algorithm}
\small
\DontPrintSemicolon
\SetKwData{Left}{left}\SetKwData{This}{this}\SetKwData{Up}{up}
\SetKwFunction{Union}{Union}\SetKwFunction{FindCompress}{FindCompress}
\SetKwInOut{Input}{input}\SetKwInOut{Output}{output}
\Input{A cost function $f(\textbf{x})$, the space for constraint set $\Omega$, and a desired precision $\epsilon$}
\Output{The optimal decision variable vector $\textbf{x}^{*}$, and the optimal value of function $f(\textbf{x}^{*})$}
\BlankLine
\emph{Initialize $f(\textbf{x}^{(0)})$ randomly and $i=1$}\;
\emph{Initialize the precision variable with $p=1$}\;
\emph{Declare the auxiliary variables $\textbf{x}$ as non-deterministic integer variables}\;
\While{$p\leq \epsilon$}{
\emph{Define bounds for $\textbf{x}$ with the \texttt{ASSUME} directive}, such that $\textbf{x} \in \Omega^\eta$\;
\emph{Describe a model for $f(\textbf{x})$}\;
\Do{$\neg l_{optimal}$ is satisfiable}{
\emph{Constrain $f(\textbf{x}^{(i)})<f(\textbf{x}^{(i-1)})$ with the \texttt{ASSUME} directive}\;
\emph{Verify the satisfiability of $l_{optimal}$ given by Eq.~\eqref{eq:optimal_literal} with the \texttt{ASSERT} directive}\;
\emph{Update $\textbf{x}^{*}=\textbf{x}^{(i)}$ and $f(\textbf{x}^{*})=f(\textbf{x}^{(i)})$ based on the counterexample}\;
\emph{Do~$i=i+1$}\;
}
\emph{Update the precision variable p}\;}
\emph{$\textbf{x}^{*}=\textbf{x}^{(i-1)}$ and $f(\textbf{x}^{*})=f^{(i-1)}(\textbf{x})$}\;
\Return{$\textbf{x}^{*}$ and $f(\textbf{x}^{*})$}\;
\caption{CEGIO: the generalized algorithm.}
\label{alg:optalg1}
\end{algorithm}

Alg.~\ref{alg:optalg1} repeats the specification and verification steps, described in Section~\ref{sec:smtoptmodel}, until the optimal solution $\textbf{x}^{*}$ is found. The precision of optimal solution defines the desired precision variable $\epsilon$. An unitary value of $\epsilon$ results in integer solutions. Solution with one decimal place is obtained for $\epsilon=10$, two decimal places are achieved for $\epsilon=100$, {\it i.e.}, the number of decimal places $\eta$ for the solution is calculated by means of the equation
\begin{equation}
\label{eq:decimal_places}
\eta=\log{\epsilon}.
\end{equation}

After the variable initialization and declaration (lines 1-3 of Alg.~\ref{alg:optalg1}), 
the search domain $\Omega$ is specified in line 5, which is defined by lower and upper bounds 
of the \textbf{x} variable, and in line 6, the model for function, $f(x)$, is defined. 
The specification step (line 8) is executed for each iteration until the desired precision is achieved. 
In this specific step, the search-space is remodelled for the $i$-th precision and it employs previous 
results of the optimization process, \textit{i.e.}, $f(\textbf{x}^{(i-1)})$. 
The verification step is performed in lines 9-10, where the candidate function $f_c$, \textit{i.e.}, 
$f(\textbf{x}^{(i-1)})$ is analyzed by means of the satisfiability check of $\neg l_{optimal}$. 
If there is a $f(\textbf{x})\leq f_{c}$ that violates the \texttt{ASSERT} directive, then the candidate 
function is updated and the algorithm returns to the specification step (line 8) to remodel the state-space again. 
If the \texttt{ASSERT} directive is not violated, the last candidate $f_{c}$ is the minimum value with the 
precision variable $p$ (initially equal to 1), thus $p$ is multiplied by $10$, adding a decimal place to the optimization solution, 
and the outer loop (\texttt{while}) is repeated.

Note that Alg.~\ref{alg:optalg1} contains two nested loops, the outer (\texttt{while}) loop is related to the 
desired precision and the inner (\texttt{do-while}) loop is related to the specification and verification steps. 
This configuration speeds-up the optimization problem due to the complexity reduction if compared to the algorithm 
originally presented in Ara\'{u}jo {\it et al.}~\cite{Araujo2016}. The generalized CEGIO algorithm uses the manipulation 
of fixed-point number precision to ensure the optimization convergence.

\subsection{Proof of Convergence}
\label{ssec:proof}

A generic optimization problem described in the previous section is formalized as follow: given a set $\Omega \subset \mathbb{R}^n$, determine $\textbf{x}^{*} \in \Omega$, such that, $f(\textbf{x}^{*})\in\Phi$ is the lowest value of the function $f$, \textit{i.e.}, $\min f(\textbf{x})$, where $\Phi\subset\mathbb{R}$ is the image set of $f$ ({\it i.e.}, $\Phi=Im(f)$). Our approach solves the optimization problem with $\eta$ decimal places, {\it i.e.}, the solution $\textbf{x}^{*}$ is an element of the rational domain $\Omega^\eta \subset \Omega$ such that 
$\Omega^\eta=\Omega \cap \Theta$, where $\Theta=\{ \textbf{x} \in \mathbb{Q}^n \vert \textbf{x} = k\times 10^{-\eta}, \forall k \in \mathbb{Z} \}$, {\it i.e.}, $\Omega^\eta$ is composed by rationals with $\eta$ decimal places in $\Omega$ ({\it e.g.}, $\Omega^0 \subset \mathbb{Z}^n$). Thus, $\textbf{x}^{*,\eta}$ is the minima of function $f$ in $\Omega^\eta$. 

\begin{mylemma}
\label{lemma:smt}
Let $\Phi$ be a finite set composed by all values $f(\textbf{x})<f_{c}$, where $f_{c}\in\Phi$ is any minimum candidate 
and $\textbf{x}\in\Omega$. The literal $\neg l_{optimal}$ (Eq.~\ref{eq:optimal_literal}) is UNSAT {\it iff} $f_{c}$ holds 
the lowest values in $\Phi$; otherwise, $\neg l_{optimal}$ is SAT {\it iff} there exists any $\textbf{x}_i \in \Omega$ 
such that $f(\textbf{x}_i)<f_{c}$.
\end{mylemma}

\begin{mytheorem}
\label{theo:minimum}
Let $\Phi_{i}$ be the $i$-th image set of the optimization problem constrained by $\Phi_{i}=\{f(\textbf{x})<f_{c}^{i}\}$, where $f_{c}^{i}=f(\textbf{x}^{(i-1)}),\forall i>0$, and $\Phi_{0}=\Phi$. There exists an $i^{*}>0$, such that $\Phi_{i^{*}}=\emptyset$, and $f(\textbf{x}^{*})=f_{c}^{i^{*}}$.
\end{mytheorem}
\begin{proof}
 Initially, the minimum candidate $f_c^{0}$ is chosen randomly from $\Phi_{0}$. 
 Considering Lemma \ref{lemma:smt}, if $\neg l_{optimal}$ is SAT, any $f(\textbf{x}^{0})$ 
 (from the counterexample) is adopted as next candidate solution ({\it i.e.}, $f_{c}^{1}=f(\textbf{x}^{(0)})$, 
 and every element from $\Phi_{1}$ is less than $f_{c}^1$. Similarly in the next iterations, 
 while $\neg l_{optimal}$ is SAT, $f_{c}^{i}=f(\textbf{x}^{(i-1)})$, and every element from $\Phi_{i}$ is less than 
 $f_c^{i}$, consequently, the number of elements of $\Phi_{i-1}$ is always less than that 
 of $\Phi_{i}$. Since $\Phi_{0}$ is finite, in the $i^{*}$-th iteration, $\Phi_{i^{*}}$ will be empty and 
 the $\neg l_{optimal}$ is UNSAT, which leads to (Lemma~\ref{lemma:smt}) $f(\textbf{x}^{*})=f_{c}^{i*}$. 
\end{proof}

Theorem~\ref{theo:minimum} provides sufficient conditions for the global minimization over a finite set; 
it solves the optimization problem defined at the beginning of this section, {\it iff} the search domain 
$\Omega^{\eta}$ is finite. It is indeed finite, once it is defined as an intersection between a bounded set 
($\Omega$) and a discrete set ($\Theta$). Thus, the CEGIO-G algorithm will always provide the minimum 
$\textbf{x}^{*}$ with $\eta$ decimal places ({\it i.e.}, $\textbf{x}^{*,\eta}$).

\subsubsection{Avoiding the Local Minima}
\label{ssec:local}

As previously mentioned, an important feature of this proposed CEGIO method is always to find the global minimum (cf. Theorem \ref{theo:minimum}). Many optimization algorithms might be trapped by local minima and they might incorrectly solve optimization problems. However, the present technique ensures the avoidance of those local minima, through the satisfiability checking, which is performed by successive SMT queries. This property is maintained for any class of functions and for any initial state.

Figures~\ref{fig:ursem03_plan1} and~\ref{fig:ursem03_plan2} show the aforementioned property of this algorithm, comparing its performance to the genetic algorithm. In those figures, Ursem03's function is adapted for a single-variable problem over $x_{1}$, {\it i.e.}, $x_{2}$ is considered fixed and equals to $0.0$, and the respective function is reduced to a plane crossing the global optimum in $x_{1}= -3$. The partial results after each iteration are illustrated by the various marks in these graphs. Note that the present method does not present continuous trajectory from the initial point to the optimal point; however, it always achieves the correct solution. Fig.~\ref{fig:ursem03_plan1} shows that both techniques (GA and SMT) achieve the global optimum. However, Fig.~\ref{fig:ursem03_plan2} shows that GA might be trapped by the local minimum for a different initial point. In contrast, the proposed CEGIO method can be initialized further away from the global minimum and as a result it can find the global minimum after some iterations, as shown in Figures~\ref{fig:ursem03_plan1} and~\ref{fig:ursem03_plan2}.
\begin{figure}[ht]
\centering
\includegraphics[width=\columnwidth]{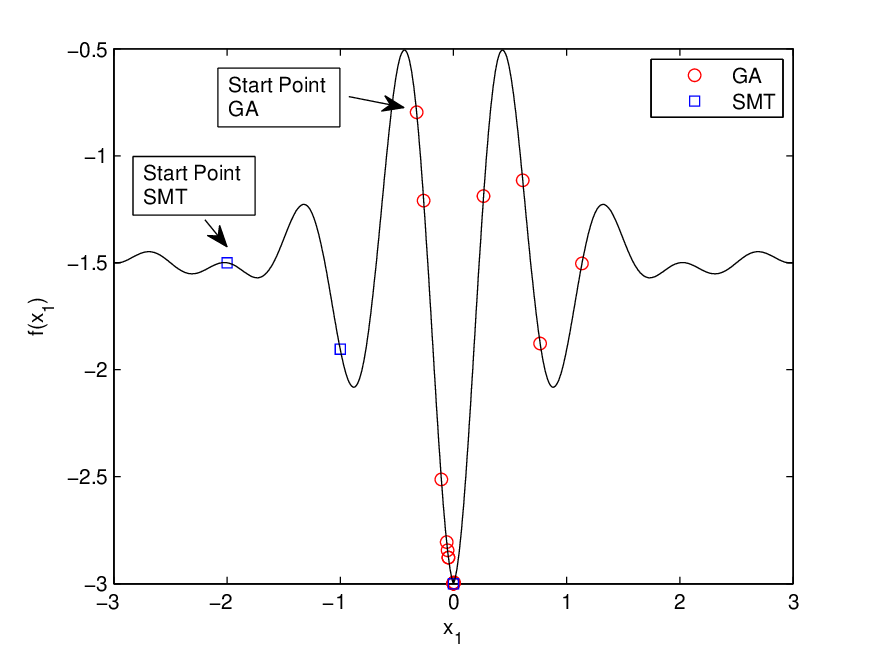}
\caption{Optimization trajectory of GA and SMT for a Ursem03's plane in $x_{2}=0$. Both methods obtain the correct answer.}
\label{fig:ursem03_plan1}
\end{figure}

\begin{figure}[ht]
\centering
\includegraphics[width=\columnwidth]{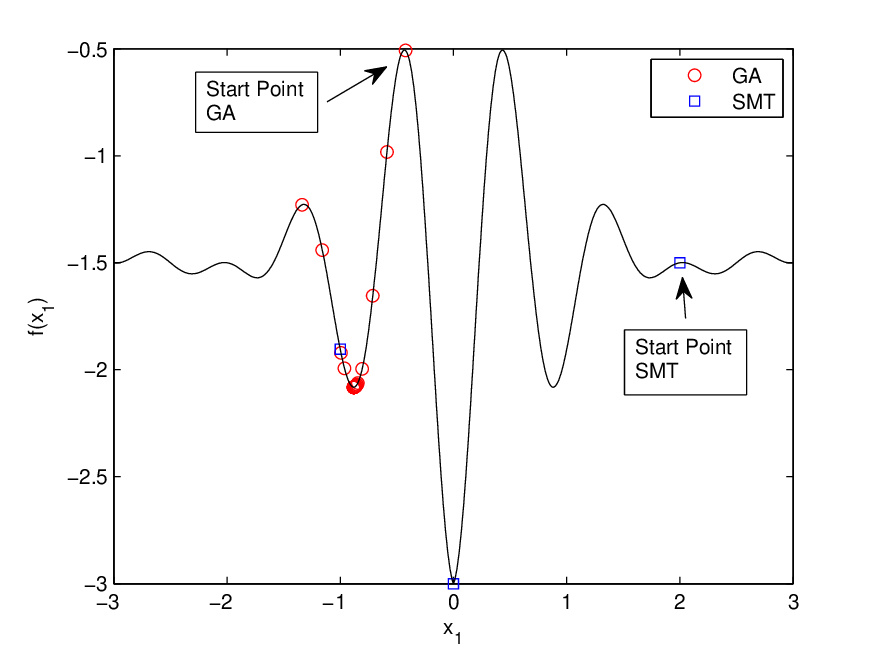}
\caption{Optimization trajectory of GA and SMT for a Ursem03's plane in $x_{2}=0$. GA is trapped by an local minimum, but SMT obtains the correct answer.}
\label{fig:ursem03_plan2}
\end{figure}

\subsection{A Simplified Algorithm for CEGIO (CEGIO-S)}
\label{ssec:simplified}

Alg.~\ref{alg:optalg1} is suitable for any class of functions, but there are some particular functions that contain further knowledge about their behaviour ({e.g.}, positive semi-definite functions such as $f(\textbf{x}) \geq 0$). Using that knowledge, Alg.~\ref{alg:optalg1} is slightly modified for handling this particular class of functions. This algorithm is named here as ``Simplified CEGIO algorithm'' (CEGIO-S) and it is presented in Alg.~\ref{alg:optalg2}.

Note that Alg. \ref{alg:optalg2} contains three nested loops after the variable initialization and declaration (lines 1-4), which is similar to the algorithm presented in~\cite{Araujo2016}. In each execution of the outer loop \texttt{while} (lines 5-25), the bounds and precision are updated accordingly. The main difference in this algorithm w.r.t the Alg.~\ref{alg:optalg1} is the presence of the condition in line $9$, \textit{i.e.}, it is not necessary to generate new checks if that condition does not hold, since the solution is already at the minimum limit, {\it i.e.}, $f(\textbf{x}^*)=0$.

Furthermore, there is another inner loop \texttt{while} (lines 12-15), which is responsible for generating multiple VCs through the \texttt{ASSERT} directive, using the interval between $f_m$ and $f(\textbf{x}^{(i-1)})$. Note that this loop generates $\alpha+1$ VCs through the step defined by $\delta$ in line 8. 

\begin{algorithm}
\small
\DontPrintSemicolon
\SetKwData{Left}{left}\SetKwData{This}{this}\SetKwData{Up}{up}
\SetKwFunction{Union}{Union}\SetKwFunction{FindCompress}{FindCompress}
\SetKwInOut{Input}{input}\SetKwInOut{Output}{output}
\Input{A cost function $f(\textbf{x})$, the space for constraint set $\Omega$, a desired precision $\epsilon$, and a learning rate $\alpha$}
\Output{The optimal decision variable vector $\textbf{x}^{*}$, and the optimal value of function $f(\textbf{x}^{*})$}
\BlankLine
\emph{Initialize $f_{m}=0$}\;
\emph{Initialize $f(\textbf{x}^{(0)})$ randomly and $i=1$}\;
\emph{Initialize the precision variable with $p=1$}\;
\emph{Declare the auxiliary variables $\textbf{x}$ as non-deterministic integer variables}\;
\While{$p\leq \epsilon$}{
\emph{Define bounds for $\textbf{x}$ with the \texttt{ASSUME} directive}, such that $\textbf{x} \in \Omega^\eta$\;
\emph{Describe a model for $f(\textbf{x})$}\;
\emph{Declare $\delta=(f(\textbf{x}^{(i-1)})-f_{m})/\alpha$}\;
\If{$(f(\textbf{x}^{(i-1)})-f_{m} > 0.00001)$}{
\Do{$\neg l_{optimal}$ is satisfiable}{
\emph{Constraint $f(\textbf{x}^{(i)})<f(\textbf{x}^{(i-1)})$ with the \texttt{ASSUME} directive}\;
\While{$(f_{m} \leq f(\textbf{x}^{(i-1)})$}{
\emph{Verify the satisfiability of $l_{optimal}$ given by Eq.~\eqref{eq:optimal_literal} with the \texttt{ASSERT} directive}\;
\emph{Do~$f_{m}=f_{m}+\delta$}\;
}
\emph{Update $\textbf{x}^{*}=\textbf{x}^{(i)}$ and $f(\textbf{x}^{*})=f(\textbf{x}^{(i)})$ based on the counterexample}\;
\emph{Do~$i=i+1$}\;
}}
\Else{break}
\emph{Update the precision variable $p$}\;}
\emph{$\textbf{x}^{*}=\textbf{x}^{(i-1)}$ and $f(\textbf{x}^{*})=f^{(i-1)}(\textbf{x})$}\;
\Return{$\textbf{x}^{*}$ and $f(\textbf{x}^{*})$}\;
\caption{CEGIO: a simplified algorithm.}
\label{alg:optalg2}
\end{algorithm}

These modifications allow Alg.~\ref{alg:optalg2} to converge faster than Alg.~\ref{alg:optalg1} for the positive semi-definite functions, since the chance of a check failure is higher due to the larger number of properties. However, if $\alpha$ represents a large number, then the respective algorithm would produce many VCs, which could cause the opposite effect and even lead the verification process to exhaust the memory. 

\section{Counterexample Guided Inductive Optimization of Convex Problems}
\label{sec:smtoptconvex}

This section presents the fast CEGIO algorithm for convex optimization problems. 
Subsection \ref{ssec:convex} presents the convex optimization problems, while the fast SMT algorithm is explained 
in Subsection~\ref{ssec:fast}. Additionally, a convergence proof of the CEGIO 
convex problem is described in Subsection~\ref{ssec:proofconvex}.

\subsection{Convex Optimization Problems}
\label{ssec:convex}

Convex functions are an important class of functions commonly found in many areas of mathematics, 
physics, and engineering~\cite{li2015selected}. A convex optimization problem is similar to Eq. \eqref{eq:optproblem}, 
where $f(\textbf{x})$ is a convex function, which satisfies Eq. \eqref{eq:convex} as
\begin{equation}
\label{eq:convex}
f(\alpha x_{1} + \beta x_{2}) \leq \alpha f(x_1) + \beta f(x_2)
\end{equation}
for all $x_i \in \mathbb{R}^{n}$, with $i=1,2$ and all $\alpha$, $\beta \in \mathbb{R}$ 
with $\alpha+\beta=1$, $\alpha \geq 0$, $\beta \geq 0$.

Theorem~\ref{theo:convex} is an important theorem for convex optimization, which is
used by most convex optimization algorithms.
\begin{mytheorem}
\label{theo:convex}
A local minimum of a convex function $f$, on a convex subset, is always a global minimum of $f$ \cite{Boyd:2004:CO:993483}.
\end{mytheorem}

Here, Theorem~\ref{theo:convex} is used to ensure 
convergence of the CEGIO convex optimization algorithm presented in Subsection~\ref{ssec:fast}.

\subsection{Fast CEGIO (CEGIO-F)}
\label{ssec:fast}

Alg.~\ref{alg:optalg1} aforementioned evolves by increasing the precision of the decision variables, 
\textit{i.e.}, in the first execution of its \texttt{while} loop, the obtained global minimum is integer
since $p=1$, called $\textbf{x}^{*,0}$. Alg.~\ref{alg:optalg3} is an improved algorithm of that Alg.~\ref{alg:optalg1} 
for application in convex functions. It will be denoted here as fast CEGIO algorithm.
Note that, the only difference of Alg.~\ref{alg:optalg1} is the insertion of line $13$, which updates $\Omega^k$ 
before of $p$. For each execution of the \texttt{while} loop, the solution is optimal for precision $p$.
A new search domain  $\Omega^{k} \subset \Omega^{\eta}$ is obtained 
from a CEGIO process over $\Omega^{k-1}$, defining $\Omega^{k}$ as follows:
$\Omega^{k}=\Omega^{\eta} \cap [x^{*,k-1}-p, x^{*,k-1}+p]$, where $x^{*,k-1}$ is the solution with $k-1$ decimal places.
\begin{algorithm}
\small
\DontPrintSemicolon
\SetKwData{Left}{left}\SetKwData{This}{this}\SetKwData{Up}{up}
\SetKwFunction{Union}{Union}\SetKwFunction{FindCompress}{FindCompress}
\SetKwInOut{Input}{input}\SetKwInOut{Output}{output}
\Input{A cost function $f(\textbf{x})$, the space for constraint set $\Omega$, and a desired precision $\epsilon$}
\Output{The optimal decision variable vector $\textbf{x}^{*}$, and the optimal value of function $f(\textbf{x}^{*})$}
\BlankLine
\emph{Initialize $f(\textbf{x}^{(0)})$ randomly and $i=1$}\;
\emph{Initialize the precision variable with $p=1$}\;
\emph{Declare the auxiliary variables $\textbf{x}$ as non-deterministic integer variables}\;
\While{$p\leq \epsilon$}{
\emph{Define bounds for $\textbf{x}$ with the \texttt{ASSUME} directive}, such that $\textbf{x} \in \Omega^k$\;
\emph{Describe a model for $f(\textbf{x})$}\;
\Do{$\neg l_{optimal}$ is satisfiable}{
\emph{Constrain $f(\textbf{x}^{(i)})<f(\textbf{x}^{(i-1)})$ with the \texttt{ASSUME} directive}\;
\emph{Verify the satisfiability of $l_{optimal}$ given by Eq.~\eqref{eq:optimal_literal} with the \texttt{ASSERT} directive}\;
\emph{Update $\textbf{x}^{*}=\textbf{x}^{(i)}$ and $f(\textbf{x}^{*})=f(\textbf{x}^{(i)})$ based on the counterexample}\;
\emph{Do~$i=i+1$}\;
}
\emph{Update set $\Omega^k$}\;
\emph{Update the precision variable p}\;}
\emph{$\textbf{x}^{*}=\textbf{x}^{(i-1)}$ and $f(\textbf{x}^{*})=f^{(i-1)}(\textbf{x})$}\;
\Return{$\textbf{x}^{*}$ and $f(\textbf{x}^{*})$}\;
\caption{Fast CEGIO.}
\label{alg:optalg3}
\end{algorithm}

\subsection{Proof of Convergence for the Fast CEGIO Algorithm}
\label{ssec:proofconvex}

The fast CEGIO algorithm computes iteratively for every 
$\Omega^{k},  0 \geq k \leq \eta$. Theorem~\ref{theo:minimum} ensures the 
global minimization for any finite $\Omega^{k}$. The global convergence of the 
fast CEGIO algorithm is ensured {\it iff} the minima of any 
$\Omega^{k-1}$ is inside $\Omega^{k}$. It holds for the generalized algorithm 
since $\Omega^{1}\subset\Omega^{2}...\subset\Omega^{k-1}\subset\Omega^{k}$. 
However, the fast CEGIO algorithm modifies $\Omega^{k}$ boundaries 
using the $k-1$-th solution.
\begin{mylemma}
\label{lemma:convex}
Let $f:\Omega^{k}\rightarrow\mathbb{R}$ be a convex function, as $\Omega^k$ is a finite set, 
Theorem ~\ref{theo:minimum} ensures that the minimum, $x^{*,k}$ in $\Omega^k$ is a local minimum for precision $p$, 
where $k=\log p$. In addition, as $f$ is a convex function, any element $x \in \Omega^{k+1}$ outside $[x^{*,k}-p, x^{*,k}+p]$ 
has its image $f(x)>f(x^{*,k})$ ensured by Eq.~\eqref{eq:convex}.
\end{mylemma}

Lemma \ref{lemma:convex} ensures that the solution is a local minimum of $f$, 
and Theorem~\ref{theo:convex} ensures that it is a global minimum. As a result, bounds 
of $\Omega^k$ can be updated on each execution of the outer \texttt{while} loop; this modification 
considerably reduces the state-space searched by the verifier, which consequently decreases 
the algorithm execution time.

\section{Experimental Evaluation}
\label{sec:experiments}

This section describes the experiments design, execution, and 
analysis for the proposed CEGIO algorithms. 
We use the ESBMC tool as verification engine to find the optimal 
solution for a particular class of functions. We also compare the 
present approaches to other exisiting techniques, including genetic algorithm, 
particle swarm, pattern search, simulated annealing, and nonlinear programming. 
Preliminary results allowed us to improve the experimental evaluation as follows.

\begin{enumerate}[(i)]
\item There are functions with multiplication operations and large inputs, 
which lead to overflow in some particular benchmarks. Thus, the data-type \textit{float} 
is replaced by \textit{double} in some particular functions to avoid overflow. 

\item ESBMC uses different SMT solvers to perform program verification. 
Depending on the selected solver, the results, verification time, and counterexamples can be 
different. This is observed in several studies~\cite{TrindadeDAES2016,Abreu2016,Bessa2016,Pereira:2016:VCP:2851613.2851830}; 
as a result, our evaluation here is also carried out using different SMT solvers such as
Boolector~\cite{boolector}, Z3~\cite{z3}, and MathSAT~\cite{mathsat5}, in order to check whether 
a particular solver heavily influences the performance of the CEGIO algorithms.

\item There are functions that present properties which permits the formulation of invariants 
to prune the state-space search, {\it e.g.}, functions that use absolute value operators 
(or polynomial functions with even degree); those functions will always present positive values. 
As a result, the optimization processes can be simplified, reducing the 
search domain to positive regions only. Such approach led to the development of Algorithm~\ref{alg:optalg2}, 
which aims to reduce the verification time.
\end{enumerate}

All experiments are conducted on a otherwise idle computer equipped with Intel 
Core i7-4790 CPU 3.60 GHz, with 16 GB of RAM, and Linux OS Ubuntu 14.10.
All presented execution times are CPU times, {\it i.e.}, only time periods spent in allocated 
CPUs, which were measured with the {\tt times} system call (POSIX system).

\subsection{Experimental Objectives} 
\label{ssec:expobjsetup}

The experiments aim to answer two research questions: 

\begin{enumerate}

\item[RQ1] \textbf{(sanity check)} what results do the proposed 
CEGIO algorithms obtain when searching for the functions optimal solution?

\item[RQ2] \textbf{(performance)} what is the proposed CEGIO algorithms 
performance if compared to genetic algorithm, particle swarm, pattern search, simulated annealing, 
and non-linear programming?

\end{enumerate}

\subsection{Description of Benchmarks}
\label{ssec:benchmarks}

In order to answer these research questions, we consider $30$ reference functions 
of global optimization problems extracted from the literature~\cite{DBLP:journals/corr/JamilY13};
all reference functions are multivariable with two decision variables. Those functions present 
different formats, {\it e.g.}, polynomials, sine, cosine, floor, sum, square root; and can be continuous, 
differentiable, separable, non-separable, scalable, non-scalable, uni-modal, and multi-modal.
The employed benchmark suite is described in Table~\ref{table:test_suite} as follows: 
benchmark name, domain, and global minimum, respectively.
\begin{table}[ht!]
\centering
\caption{Benchmark Suite for Global Optimization Problems.}
\label{table:test_suite}
\begin{tabular}{|c|c|c|c|c|c|c|c|c|c|c|c|c|}
\hline
  \#	& Benchmark		& Domain 			& Global Minima 					\\ \hline 

  1 	& Alpine 1		& $-10 	\leq x_i \leq 10$ 	& $f(0,0)=0$ 						\\ \hline  
  2	& Booth			& $-10 	\leq x_i \leq 10$ 	& $f(1,3)=0$ 						\\ \hline  
  3	& Chung			& $-10 	\leq x_i \leq 10$ 	& $f(0,0)=0$ 						\\ \hline  
  4	& Cube			& $-10 	\leq x_i \leq 10$ 	& $f(1,1)=0$ 						\\ \hline  
  5 	& Dixon \& Price	& $-10 	\leq x_i \leq 10$ 	& $f(xi)=0, xi= 2^{-((2i-2) /2i)}$ 			\\ \hline
  6	& Egg Crate		& $-5 	\leq x_i \leq 5$ 	& $f(0,0)=0$						\\ \hline
  7	& Himmeblau		& $-5 	\leq x_i \leq 5$	& $f(3,2)=0$ 						\\ \hline
  8	& Leon			& $-2 	\leq x_i \leq 2$	& $f(1,1)=0$						\\ \hline
  9	& Power Sum		& $-1 	\leq x_i \leq 1$	& $f(0,0)=0$ 						\\ \hline
\multirow{2}{*}{10} & \multirow{2}{*}{Price 4} & \multirow{2}{*}{$-10 	\leq x_i \leq 10$} & $f\{(0,0),(2,4),$       \\
                   &                      &                       & $(1.464,-2.506)\}=0$                            	\\ \hline
  11	& Engvall		& $-10 	\leq x_i \leq 10$ 	& $f(1,0)=0$ 							\\ \hline
  12	& Schumer		& $-10 	\leq x_i \leq 10$ 	& $f(0,0)=0$ 						\\ \hline
  13	& Tsoulos		& $-1 	\leq x_i \leq 1$ 	& $f(0,0)=-2$ 		\\   \hline

\multirow{2}{*}{14} & \multirow{2}{*}{Branin RCOS} & \multirow{2}{*}{$-5 \leq x_i \leq 15$} & $f\{(-\pi,12.275),(\pi,2.275),$       \\
                   &                      &                       & $(3\pi,2.425)\}=0.3978873$                            	\\ \hline

  15	& Schuwefel 2.25	& $-10 	\leq x_i \leq 10$ 	& $f(1,1)=0$ 		\\   \hline
  16	& Sphere		& $0 	\leq x_i \leq 10$ 	& $f(0,0)=0$		\\   \hline
  17	& Step 2		& $-100 \leq x_i \leq 100$	& $f(0,0)=0$ 		\\	\hline 
  18  	& Scahffer 4		& $-10 	\leq x_i \leq 10$	& $f(0, 1.253)=0.292$		\\ \hline
  19	& Sum Square		& $-10 	\leq x_i \leq 10$	& $f(0,0)=0$ 		\\ \hline
  20	& Wayburn Seader 2	& $-500 \leq x_i \leq 500$	& $f (\{0.2, 1\},\{0.425, 1\})$		\\ \hline
  21	& Adjiman		& $-1 	\leq x_i \leq 2$ 	& $f(2,0.10578)= -2.02181$ 	\\   \hline
  22	& Cosine		& $-1 	\leq x_i \leq 1$ 	& $f(0,0)=-0.2$ 		\\   \hline
  23	& S2			& $-5 	\leq x_i \leq 5$ 	& $f(x_1,0.7)=2$ 		\\   \hline
  24	& Matyas		& $-10 	\leq x_i \leq 10$ 	& $f(0,0)=0$ 		\\ \hline  
  25	& Rotated Ellipse	& $-500 \leq x_i \leq 500$ 	& $f(0,0)=0$ 		\\   \hline
  26	& Styblinski Tang	& $-5 	\leq x_i \leq 5$ 	& $f(â2.903,â2.903)= -78.332$	\\ \hline
  27	& Trecanni		& $-5 	\leq x_i \leq 5$	& $f(\{0,0\},\{â2,0\})=0$ 		\\ \hline 
  28	& Ursem 1		& $-3	\leq x_i \leq 3$	& $f(1.697136,0)= -4.8168$	\\ \hline
  29	& Zettl			& $-5 	\leq x_i \leq 10$	& $f (â0.029, 0) = - 0.0037$ 		\\ \hline
  30 	& Zirilli		& $-10 	\leq x_i \leq 10$	& $f(â1.046, 0) \approx	 - 0.3523$	\\ \hline
    
\end{tabular}
\end{table}

In order to perform the experiments with three different CEGIO algorithms, 
generalized (Alg.~\ref{alg:optalg1}), simplified (Alg.~\ref{alg:optalg2}), and fast (Alg.~\ref{alg:optalg3}), 
a set of programs were developed for each function, taking into account each algorithm and 
varying the solver and the data-type accordingly. For the experiment with the generalized algorithm, 
all benchmarks are employed; for the simplified algorithm, $15$ functions are selected 
from the benchmark suite. By previous observation, we can affirm that those $15$ functions 
are semi-definite positive; lastly, we selected $10$ convex functions 
from the benchmark suite to evaluate the fast algorithm.

For the experiments execution with the proposed algorithms, random values are generated, 
belonging to the solutions space of each function, and they are used as initialization 
of the proposed algorithms, as described in Section~\ref{sec:smtoptnonconvex}. The other optimization 
techniques used for comparison, had all benchmarks performed by means of the Optimization 
Toolbox in MATLAB 2016b~\cite{optoolboxmanual} with the entire benchmark suite. 
The time presented in the following tables are related to the average of $20$ executions for each benchmark; 
the measuring unit is always in seconds based on the CPU time.

\subsection{Experimental Results}
\label{ssec:results}

In the next subsections, we evaluate the proposed CEGIO 
algorithms performance; we also compare them to other traditional techniques.

\subsubsection{Generalized Algorithm (CEGIO-G) Evaluation}
\label{ssec:results_generic}

The experimental results presented in Table~\ref{table:test_alg1} 
are related to the performance evaluation of the Generalized Algorithm (CEGIO-G) 
(cf. Alg.~\ref{alg:optalg1}). Here, the CPU time is measured in seconds 
to find the global minimum using the ESBMC tool with a particular SMT solver. 
Each column of Table~\ref{table:test_alg1} is described as follows: 
columns 1 and 5 are related to functions of the benchmark suite; 
columns 2 and 6 are related to the configuration of ESBMC with Boolector;
columns 3 and 7 are related to ESBMC with Z3; 
and columns 4 and 8 are related to ESBMC with MathSAT.
\begin{table}[]
\centering
\caption{Experimental Times Results with the Generic Algorithm (in seconds).}
\label{table:test_alg1}
\begin{tabular}{|c|c|c|c||c|c|c|c|}
\hline
\# & Boolector 		& Z3     & MathSAT   		& \# & Boolector & Z3     	& MathSAT \\ \hline
1  & \textbf{537}       & 6788   & 590     		& 16 & \textbf{1}         	& \textbf{1}      	& 2       \\ \hline
2  & 2660      		& 972    & \textbf{5} 		& 17 & 3         		& \textbf{1}      	& 11      \\ \hline
3  & 839*      & 5812*  & \textbf{2}       		& 18 & \textbf{6785}      	& 14738  		& 33897        \\ \hline
4  & 170779*   & 77684* & \textbf{5}       		& 19 & 41        		& \textbf{1}      	& 3               \\ \hline
5  & 36337*    & 22626* & \textbf{8}       		& 20 & 33794*    		& 36324  		& \textbf{37}       \\ \hline
6  & 5770      & 3565   & \textbf{500}     		& 21 & \textbf{665}       	& 2969   		& 19313    \\ \hline
7  & 4495      & 11320  & \textbf{10}      		& 22 & \textbf{393}       	& 2358   		& 3678         \\ \hline
8  & 269       & 1254   & \textbf{4}       		& 23 & 32        		& 13     		& \textbf{10}        \\ \hline
9  & \textbf{3}         & 40     & 4       		& 24 & 5945      		& 5267   		& \textbf{23}        \\ \hline
10 & 16049*    & 110591 & \textbf{6}       		& 25 & 1210      		& 2741   		& \textbf{16}       \\ \hline
11 & 1020      & 3653   & \textbf{662}    		& 26 & 1330      		& 19620  		& \textbf{438}          \\ \hline
12 & 445       & 20     & \textbf{4}       		& 27 & \textbf{76}        	& 269    		& 2876         \\ \hline
13 & \textbf{305}       & 9023   & 2865    		& 28 & 808       		& \textbf{645}    	& 11737        \\ \hline
14 & 17458     & 25941  & \textbf{3245}    		& 29 & 271       		& 611   		& \textbf{11}        \\ \hline
15 & 2972      & 5489   & \textbf{7}       		& 30 & \textbf{383}       	& 720    		& 662         \\ \hline
\end{tabular}
\end{table}

All benchmarks are employed for evaluating the generalized algorithm performance. 
The correct global minima is found in all benchmarks using different SMT solvers: 
MathSAT, Z3, and Boolector. For all evaluated benchmarks, MathSAT is $4.6$ 
times faster than Z3, although there are benchmarks in which MathSAT took longer than Z3, 
{\it e.g.}, in \textit{Adjiman} and \textit{Cosine} functions. If we compare Boolector performance
to other SMT solvers, we can also observe that it is routinely faster than both Z3 and MathSAT.

Initially, all experiments were performed using \textit{float}-type variables, 
but we noticed that there was either overflow or underflow in some particular benchmarks, 
{\it e.g.}, the \textit{Cube} functions. It occurs due to truncation in some arithmetic 
operations and series, {\it e.g.}, sines and cosines, once the verification engine employs fixed-point 
for computations. This might lead to a serious problem if there are several 
operations being performed with very large inputs, in a way that causes errors that 
can be propagated; those errors thus lead to incorrect results. For this specific reason, we decided to use 
\textit{double}-type variables for these particular benchmarks to increase precision. We observed 
that the global minimum value is always found using double precision, but it takes longer than 
using \textit{float}-type variables. The cells with asterisks in Table~\ref{table:test_alg1} 
identify the benchmarks that we use \textit{double}- instead of \textit{float}-type.


Additionally, we observed that when the function has more than one global minimum, 
{\it e.g.}, \textit{Wayburn Seader} 2 with the decision variables $f\{(0.2.1),(0.425,1)\}$, 
the algorithm first finds the global minimum with the decision variables of less precision, 
then in this case $f(0.2,1)$. Analyzing Alg.~\ref{alg:optalg1}, when an overall minimum 
value is found, the condition in line $9$ is not satisfied since there is no candidate function 
with a value less than the current one found; on line $13$ the precision is updated and the outer 
loop starts again. Even if there is another overall minimum in this new precision, it will not be 
considered by the ASSUME directive in line $8$ since the decision variables define a candidate 
function with the same value as the current function $f(x)$, and not less than the defined 
in Eq.~\ref{eq:optimal_literal}. In order to find the other global minimum, it would be necessary 
to limit it with the ASSUME directive, disregarding the previous minimum.

\subsubsection{Simplified Algorithm (CEGIO-S) Evaluation}
\label{ssec:results_simplified}

The simplified algorithm (CEGIO-S) is applied to functions that 
contain invariants about the global minimum, {\it e.g.}, semi-definite 
positive functions, where it is not needed to search for their minimum 
in the $f$ negative values. For instance, the \textit{leon} function presented in 
Eq.~\eqref{eq:function_} has the global minimum at $f(1,1)=0$ as follows 
%
\begin{equation}
	\label{eq:function_}
	f(x_1,x_2) = 100(x_2 -{x_1}^2)^2 + (1-x_1)^2.
\end{equation}

By inpesction it is possible to claim that there are 
no negative values for $f(x)$. Therefore, in order to effectively evaluate 
Algorithm~\ref{alg:optalg2}, $15$ benchmarks are selected, 
which have modules or exponential pair, {\it i.e.}, the lowest possible value to 
global minimum is a non-negative value. The experiments are performed using the 
\textit{float} data-type, and \textit{double} as needed to avoid overflow, 
using the same solvers as described in Subsection~\ref{ssec:results_generic}. 
According to the experimental results shown in Table~\ref{table:test_alg2}, we 
confirmed that all obtained results match those described in the literature \cite{DBLP:journals/corr/JamilY13}.
\begin{table}[H]
\centering
\caption{Experimental Results with the Simplified Algorithm (in seconds).}
\label{table:test_alg2}
\begin{tabular}{|c|c|c|c|c|c|c|}
\hline
\multirow{2}{*}{\#} & \multicolumn{3}{c|}{CEGIO-S}           & \multicolumn{3}{c|}{CEGIO-G}  \\ \cline{2-7} 
                    & Boolector  & Z3         & MathSAT    & Boolector & Z3    	& MathSAT \\ \hline
1                   & 74         & \textbf{2}          & 413        & 537       & 6788 	& 590     \\ \hline
2                   & \textbf{\textless1} & \textbf{\textless1} & 1          & 2660      & 972    & 5       \\ \hline
3                   & \textbf{\textless1} & \textbf{\textless1} & \textbf{\textless1} & 839*      & 5812*  & 2       \\ \hline
4                   & \textbf{\textless1} & \textbf{\textless1} & 2          & 170779*   & 77684* & 5       \\ \hline
5                   & 14         & \textbf{2}          & 6          & 36337*    & 22626* & 8       \\ \hline
6                   & 34         & \textbf{2}          & 240        & 5770      & 3565   & 500     \\ \hline
7                   & \textbf{1}          & \textbf{1}          & 6          & 4495      & 11320  & 10      \\ \hline
8                   & 1          & \textbf{\textless1} & 2          & 269       & 1254   & 4       \\ \hline
9                   & \textbf{\textless1} & \textbf{\textless1} & 2          & 3         & 40     & 4       \\ \hline
10                  & \textbf{\textless1} & \textbf{\textless1} & 5          & 16049*     & 110591 & 6       \\ \hline
12                  & \textbf{\textless1} & \textbf{\textless1} & 1          & 445       & 20     & 4       \\ \hline
15                  & \textbf{1}          & 2          & 5          & 2972      & 5489   & 7       \\ \hline
16                  & \textbf{\textless1} & \textbf{\textless1} & \textbf{\textless1} & 1         & 1     	& 2       \\ \hline
19                  & \textbf{\textless1} & \textbf{\textless1} & \textbf{\textless1} & 41        & 1      & 3       \\ \hline
20                  & 215        & 2446       & \textbf{30}         & 33794*     & 36324  & 37      \\ \hline
\end{tabular}
\end{table}

Additionally, we can see that the simplified algorithm reduces 
the optimization time considerably, with particular benchmarks reaching 
less than $1$ second. However, the reduction with the MathSAT solver is less 
expressive since it models \textit{float}-type variables using floating-point 
arithmetic in both CEGIO-S and CEGIO-G algorithms, while Boolector and Z3 uses 
fixed-point arithmetic. We conclude that either our fast algorithm is suitable
for fixed-point architectures or MathSAT implements more aggressive
simplifications than Boolector and Z3 


The purpose of this algorithm is to find the global minimum 
to reduce the verification time, for functions that have invariants 
about the global minimum. However,  the simplified algorithm 
run-time might be longer than the generalized one since 
it requires parameter settings according to the function. 
As described in Subsection~\ref{ssec:simplified}, in line $8$ of Algorithm~\ref{alg:optalg2}, 
we have the variable $\delta$ that defines the state-space search segmentation; 
$\delta$ is obtained by the difference of the current $f(x)$ and the boundary that we know, 
divided by the variable $\alpha$ (previously established). If we have a very large absolute value
for $\alpha$, then we would have additional checks, thus creating many 
more properties to be checked by the verifier (and thus leading it to longer verification times).

If we analyze function S2 in Eq.~\eqref{eq:function_S2}, then we can easily inspect 
that there is no $f(x)$ less than $2$; in this case, therefore, 
in line $1$ of Algorithm~\ref{alg:optalg2}, one can establish $f_m$ with the value 
$2$. This slightly change in the initialization of $f_m$ in Algorithm~\ref{alg:optalg2} 
prunes the state-space search and the verification time accordingly.

\begin{equation}
	\label{eq:function_S2}
	f(x_1,x_2) = 2 + (x_2-0.7)^2
\end{equation}

\subsubsection{Fast Algorithm (CEGIO-F) Evaluation}
\label{ssec:results_fast}

The experimental results for the fast algorithm (CEGIO-F) are 
presented in Table~\ref{table:test_alg3}. This algorithm is
applied to convex functions, where there is only a global minimum; 
in particular, the state-space is reduced on each iteration of the 
\textit{while}-loop in Algorithm~\ref{alg:optalg3}, ensuring that the 
global minimum is in the new (delimited) space, and then it performs 
a new search in that space to reduce the overall optimization time.

In order to evaluate the effectiveness of Algorithm~\ref{alg:optalg3}, 
we selected approximately $10$ convex functions of the benchmark suite; 
we also compare the fast algorithm (CEGIO-F) results with the generalized 
one (CEGIO-G). We observed that there are significant performance improvements 
if we compare CEGIO-F to CEGIO-G for convex function benchmarks, {\it i.e.}, CEGIO-F algorithm 
is $1000$ times faster using the SMT solver Boolector and $750$ times faster using the 
SMT solver Z3 than the (original) CEGIO-G algorithm, as shown in Table~\ref{table:test_alg3}.
\begin{table}[H]
\centering
\caption{Experimental Results with the Fast Algorithm (in seconds). }
\label{table:test_alg3}\begin{tabular}{|c|c|c|c|c|}
\hline
\multirow{2}{*}{\#} & \multicolumn{2}{c|}{CEGIO-F} & \multicolumn{2}{c|}{CEGIO-G} \\ \cline{2-5} 
                    & Boolector    & Z3          & Boolector      & Z3        \\ \hline
2                   & \textbf{\textless1}   	& \textbf{\textless1}  & 2660           & 972     \\ \hline
3                   & 33*          & \textbf{26*}         & 839*           & 5812*        \\ \hline
4                   & 43*          & \textbf{25}          & 170779*        & 77684*     \\ \hline
5                   & 59*          & \textbf{10*}         & 36337*         & 22626*      \\ \hline
9                   & \textbf{1}            & 10          & 3         	  & 40      \\ \hline
12                  & \textbf{1}            & 2           & 445            & 20      \\ \hline
19                  & 1            & \textbf{\textless1}  & 41             & 1       \\ \hline
24                  & 7            & \textbf{2}           & 5945           & 5267        \\ \hline
25                  & 2            & \textbf{1*}          & 1210           & 2741        \\ \hline
29                  & \textbf{63*}          & 76          & 271            & 611       \\ \hline 
\end{tabular}
\end{table}

\subsubsection{Comparison to Other Traditional Techniques}
\label{ssec:comparison}

In this section, our CEGIO algorithms are compared 
to other traditional optimization techniques: genetic algorithm (GA), 
particle swarm (ParSwarm), pattern search (PatSearch), simulated annealing (SA), 
and nonlinear programming (NLP).

Table~\ref{tb:comparison} describes the hit rates and the mean time for each 
function w.r.t. our proposal (ESBMC) and other existing techniques (GA, ParSwarm, PatSearch, SA, and NLP). 
An identification for each algorithm is defined: (1) Generic, (2) Simplified, and (3) Fast. 
All traditional optimization techniques are executed $20$ times using MATLAB, for  obtaining 
the correctness rate and the mean time for each function. 

Our hit rate is omitted for the sake of space, but our algorithms have found 
the correct global minima in 100\% of the experiments. The experiments show 
that our hit rate is superior than any other optimization technique, although 
the optimization time is usually longer. 

The other optimization techniques are very sensitive to non-convexity; 
for this reason, they are usually trapped by local minima. The other optimization 
techniques presented better performance in convex functions. Specifically, they converge 
faster to the response and there are no local minimums that could lead to incorrect results,
whereas with the non-convex functions, their hit rate is lower, precisely because there are local minimums.
\begin{table}[H]
\centering
\caption{Experimental Results with the Traditional Techniques and our Best CEGIO Algorithm (in seconds).}
\label{tb:comparison}
\begin{tabular}{|c|c|c|c|c|c|c|c|c|c|c|c|}
\hline
\multirow{2}{*}{\#} & ESBMC 		& \multicolumn{2}{c|}{GA} & \multicolumn{2}{c|}{ParSwarm} & \multicolumn{2}{c|}{PatSearch} & \multicolumn{2}{c|}{SA} & \multicolumn{2}{c|}{NLP} \\ \cline{2-12} 
                    & T  		& R$\%$        	& T   & R$\%$           & T        & R$\%$          & T          & R\%         & T        & R\%         & T        \\ \hline
1                   & $2^{(2)}$     	& 25  		& 1               & 45                      & 3                     & 10                      & 4                     & 50                      & 1                         & 0                       & 9                         \\ \hline
2                   & $<1^{(2,3)}$	& 100 		& 10                    & 100                     & 2                     & 100                     & 6                     & 95                      & 1                         & 100                     & 2                         \\ \hline
3                            & $<1^{(2)}$       	& 100 & 9                     & 100                     & 1                     & 100                     & 4                     & 90                      & 1                         & 100                     & 5                         \\ \hline
4                             & $<1^{(2)}$       	& 20  & 1                     & 30                      & 3                     & 0                       & 8                     & 10                      & 2                     & 100                     & 7                     \\ \hline
5                             & $2^{(2)}$		& 0   & 9                     & 0                       & 2                     & 0                       & 3                     & 0                       & 1                     & 0                       & 2                     \\ \hline
6                             & $2^{(2)}$        	& 100 & 9                     & 100                     & 1                     & 70                      & 3                     & 100                     & 1                     & 25                      & 2                     \\ \hline
7                             & $1^{(2)}$        	& 60  & 9                     & 50                      & 1                     & 25                      & 3                     & 15                      & 1                     & 35                      & 2                     \\ \hline
8                           		& $<1^{(2)}$   	& 90  & 1                     & 75                      & 2                     & 0                       & 7                     & 10                      & 1                     & 100                     & 4                     \\ \hline
9                            & $<1^{(2)}$       	& 100 & 9                     & 100                     & 1                     & 100                     & 3                     & 50                      & 1                     & 100                     & 2                     \\ \hline
10                           & $<1^{(2)}$       	& 0   & 9                     & 10                      & 2                     & 0                       & 7                     & 0                       & 4                     & 50                      & 2                     \\ \hline
11                          & $662^{(1)}$          	& 90   & 1                     & 100                       & 2                     & 90                       & 3                     & 95                       & 1                     & 100                       & 7                     \\ \hline
12                           & $<1^{(2)}$       	& 100 & 9                     & 100                     & 1                     & 100                     & 4                     & 75                      & 1                     & 100                     & 4                     \\ \hline
13                            & $3^{(1)}$       	& 100 & 9                     & 95                      & 1                     & 100                     & 3                     & 75                      & 9                     & 0                       & 6                     \\ \hline
14                         & $32^{(1)}$      	& 100 & 8                     & 100                     & 9                     & 100                     & 4                     & 75                      & 8                     & 0                       & 5                     \\ \hline
15                           & $1^{(1)}$       	& 100 & 1                     & 95                      & 1                     & 100                     & 3                     & 100                     & 1                     & 100                     & 2                     \\ \hline
16                            & $<1^{(2)}$       	& 100 & 10                    & 100                     & 7                     & 100                     & 4                     & 100                     & 1                     & 100                     & 2                     \\ \hline
17                            & $1^{(1)}$    		& 0   & 9                     & 0                       & 1                     & 0                       & 2                     & 0                       & 8                     & 0                       & 1                     \\ \hline
18                            & $1^{(1)}$ 	  	& 30   & 1         & 15  & $<1$          & 0     & $<1$                 & 0              & 2                     & 0                       & $<1$                     \\ \hline
19                          & $<1^{(2,3)}$     & 100 & 9                     & 100                     & 1                     & 100                     & 4                     & 100                     & 1                     & 100                     & 2                     \\ \hline
20                         & $30^{(2)}$       & 45  & 10                    & 45                      & 2                     & 0                       & 8                     & 50                      & 2                     & 45                      & 6                     \\ \hline
21                           & $665^{(1)}$        	& 0   & 10                    & 100                     & 1                     & 0                       & 4                     & 80                      & 2                     & 95                      & 2                     \\ \hline
22                            & $393^{(1)}$       	& 100 & 9                     & 100                     & 1                     & 95                      & 3                     & 95                      & 2                     & 15                      & 2                     \\ \hline
23                           &  $10^{(1)}$   	& 65   & $<1$                    & 100                       & $<1$                     & 100                       & $<1$                     & 85                       & 1                     & 100                       & $<1$                    \\ \hline
24                           & $<2^{(3)}$       & 100 & 9                     & 100                     & 1                     & 100                     & 8                     & 10                      & 1                     & 100                     & 2                     \\ \hline
25                          & $<1^{(3)}$       & 100 & 9                     & 100                     & 2                     & 100                     & 7                     & 100                     & 1                     & 100                     & 2                     \\ \hline
26                            & $438^{(1)}$      	& 100 & 9                     & 100                     & 1                     & 50                      & 3                     & 100                     & 1                     & 35                      & 2                     \\ \hline
27                            & $76^{(1)}$    		& 0   & 9                     & 0                       & 1                     & 0                       & 3                     & 0                       & 1                     & 0                       & 2                     \\ \hline
28                          & $645^{(1)}$       	& 100 & 9                     & 100                     & 1                     & 100                     & 3                     & 80                      & 1                     & 65                      & 2                     \\ \hline
29                            & $<63^{(3)}$    	& 100 & 9                     & 100                     & 1                     & 100                     & 4                     & 100                      & 1                     & 100                     & 3                     \\ \hline
30                            & $383^{(1)}$          	& 100 & 9                     & 100                     & 1                     & 100                     & 3                     & 60                      & 1                     & 75                      & 2                     \\ \hline
\end{tabular}
\end{table}

\section{Conclusions}
\label{sec:conc}

This paper presented three variants of a counterexample-guided inductive optimization approach
for optimizing a wide range of functions based on counterexamples extracted 
from SMT solvers. In particular, this work proposed algorithms to perform inductive generalization 
based on counterexamples provided by a verification oracle for optimizing convex and 
non-convex functions and also presented respective proofs for global convergence. 
Furthermore, the present study provided an analysis about the influence of the solver 
and data-types in the performance of the proposed algorithms.

All proposed algorithms were exhaustively evaluated using a large 
set of public available benchmarks. We also evaluated the present 
algorithms performance using different SMT solvers and compared 
them to other state-of-art optimization techniques (genetic algorithm, 
particle swarm, pattern search, nonlinear programming, and simulated annealing). 
The counterexample-guided inductive optimization algorithms are able 
to find the global optima in 100\% of the benchmarks, and the optimization time 
is significantly reduced if compared to Ara\'{u}jo {\it et al.}~\cite{Araujo2016}. 
Traditional optimization techniques are typically trapped by local minima and 
are unable to ensure the global optimization, although they still present lower 
optimization times than the proposed algorithms.

In contrast to previous optimization techniques, the present approaches are 
suitable for every class of functions; they are also complete, providing an improved accuracy 
compared to other existing traditional techniques. Future studies include the application 
of the present approach to autonomous vehicles navigation systems, enhancements 
in the model-checking procedure for reducing the verification time by means of 
multi-core verification~\cite{TrindadeSBESC2015} and invariant 
generation~\cite{Rocha2015,Gadelha2015}. We also intend to improve 
Alg.~\ref{alg:optalg2} by implementing a dynamic learning rate since 
it is currently fixed in the proposed algorithm. Finally, we intend to extend all 
present approaches for multi-objective optimization problems.




\section*{References}


\begin{thebibliography}{10}
\expandafter\ifx\csname url\endcsname\relax
  \def\url#1{\texttt{#1}}\fi
\expandafter\ifx\csname urlprefix\endcsname\relax\def\urlprefix{URL }\fi
\expandafter\ifx\csname href\endcsname\relax
  \def\href#1#2{#2} \def\path#1{#1}\fi

\bibitem{deb2004optimization}
K.~Deb, Optimization for Engineering Design: Algorithms and Examples,
  Prentice-Hall of India, 2004.

\bibitem{Gattie2007}
D.~K. Gattie, R.~C. Wicklein, Curricular value and instructional needs for
  infusing engineering design into k-12 technology education, Journal of
  Technology Education 19~(1) (2007) 13.
\newblock \href {http://dx.doi.org/10.21061/jte.v19i1.a.1}
  {\path{doi:10.21061/jte.v19i1.a.1}}.

\bibitem{Shoham2008}
Y.~Shoham, Computer science and game theory, Commun. ACM 51~(8) (2008) 74--79.
\newblock \href {http://dx.doi.org/10.1145/1378704.1378721}
  {\path{doi:10.1145/1378704.1378721}}.

\bibitem{Teich2012}
J.~Teich, Hardware/software codesign: The past, the present, and predicting the
  future, Proceedings of the IEEE 100~(Special Centennial Issue) (2012)
  1411--1430.
\newblock \href {http://dx.doi.org/10.1109/JPROC.2011.2182009}
  {\path{doi:10.1109/JPROC.2011.2182009}}.

\bibitem{kowalski1995selected}
M.~Kowalski, C.~Sikorski, F.~Stenger, Selected Topics in Approximation and
  Computation, Oxford University Press, 1995.

\bibitem{derigs2009optimization}
U.~Derigs, Optimization and Operations Research -- Volume I:, 2009.

\bibitem{garfinkel1972integer}
R.~Garfinkel, G.~Nemhauser, Integer programming, Series in decision and
  control, Wiley, 1972.

\bibitem{Bartholomew--Biggs2008}
M.~Bartholomew-Biggs, The Steepest Descent Method, Springer US, Boston, MA,
  2008, pp. 1--8.
\newblock \href {http://dx.doi.org/10.1007/978-0-387-78723-7_7}
  {\path{doi:10.1007/978-0-387-78723-7_7}}.

\bibitem{goldberg1989genetic}
D.~Goldberg, Genetic Algorithms in Search, Optimization, and Machine Learning,
  Artificial Intelligence, Addison-Wesley Publishing Company, 1989.

\bibitem{cavazzuti2012optimization}
M.~Cavazzuti, Optimization Methods: From Theory to Design Scientific and
  Technological Aspects in Mechanics, Springer Berlin Heidelberg, 2012.

\bibitem{AlurSyGus2013}
R.~Alur, R.~Bodik, G.~Juniwal, M.~M.~K. Martin, M.~Raghothaman, S.~A. Seshia,
  R.~Singh, A.~Solar-Lezama, E.~Torlak, A.~Udupa, Syntax-guided synthesis, in:
  2013 Formal Methods in Computer-Aided Design, 2013, pp. 1--8.
\newblock \href {http://dx.doi.org/10.1109/FMCAD.2013.6679385}
  {\path{doi:10.1109/FMCAD.2013.6679385}}.

\bibitem{Dorigo2006}
M.~Dorigo, M.~Birattari, T.~Stutzle, Ant colony optimization, IEEE Computat.
  Intell. Mag. 1~(4) (2006) 28--39.
\newblock \href {http://dx.doi.org/10.1109/MCI.2006.329691}
  {\path{doi:10.1109/MCI.2006.329691}}.

\bibitem{Araujo2016}
R.~Ara\'ujo, I.~Bessa, L.~Cordeiro, J.~E.~C. Filho, {SM}t-based verification
  applied to non-convex optimization problems, in: Proceedings of VI Brazilian
  Symposium on Computing Systems Engineering, 2016, pp. 1--8.
\newblock \href {http://dx.doi.org/10.1109/SBESC.2016.010}
  {\path{doi:10.1109/SBESC.2016.010}}.

\bibitem{z3}
L.~De~Moura, N.~Bj{\o}rner, Z3: An efficient {SMT} solver, in: {TACAS},
  Springer-Verlag, Berlin, Heidelberg, 2008, pp. 337--340.

\bibitem{boolector}
R.~Brummayer, A.~Biere, {Boolector: An Efficient {SMT} Solver for Bit-Vectors
  and Arrays}, in: Tools and Algorithms for the Construction and Analysis of
  Systems (TACAS), 2009, pp. 174--177.

\bibitem{mathsat5}
A.~Cimatti, A.~Griggio, B.~Schaafsma, R.~Sebastiani, The math{SAT5} {SMT}
  solver, in: Tools and Algorithms for the Construction and Analysis of
  Systems, 2013, pp. 93--107.
\newblock \href {http://dx.doi.org/10.1007/978-3-642-36742-7}
  {\path{doi:10.1007/978-3-642-36742-7}}.

\bibitem{functionslist}
M.~Jamil, X.-S. Yang, A literature survey of benchmark functions for global
  optimisation problems, {IJMMNO} 4~(2) (2013) 150--194.
\newblock \href {http://dx.doi.org/10.1504/IJMMNO.2013.055204}
  {\path{doi:10.1504/IJMMNO.2013.055204}}.

\bibitem{pswarmref:olsson2011particle}
A.~Olsson, Particle Swarm Optimization: Theory, Techniques and Applications,
  Engineering tools, techniques and tables, Nova Science Publishers, 2011.

\bibitem{patternsref:alberto2004pattern}
P.~Alberto, F.~Nogueira, H.~Rocha, L.~N. Vicente, Pattern search methods for
  user-provided points: Application to molecular geometry problems, SIAM
  Journal on Optimization 14~(4) (2004) 1216--1236.

\bibitem{saref:Laarhoven:1987:SAT:59580}
P.~J.~M. Laarhoven, E.~H.~L. Aarts (Eds.), Simulated Annealing: Theory and
  Applications, Kluwer Academic Publishers, Norwell, MA, USA, 1987.

\bibitem{npref:byrd2000trust}
R.~H. Byrd, J.~C. Gilbert, J.~Nocedal, A trust region method based on interior
  point techniques for nonlinear programming, Mathematical Programming 89~(1)
  (2000) 149--185.

\bibitem{Nieuwenhuis2006}
R.~Nieuwenhuis, A.~Oliveras, On SAT Modulo Theories and Optimization Problems,
  Springer Berlin Heidelberg, Berlin, Heidelberg, 2006, pp. 156--169.
\newblock \href {http://dx.doi.org/10.1007/11814948_18}
  {\path{doi:10.1007/11814948_18}}.

\bibitem{Eldib2014}
H.~Eldib, C.~Wang, An {SMT} based method for optimizing arithmetic computations
  in embedded software code, IEEE CAD 33~(11) (2014) 1611--1622.
\newblock \href {http://dx.doi.org/10.1109/TCAD.2014.2341931}
  {\path{doi:10.1109/TCAD.2014.2341931}}.

\bibitem{Estrada2003}
G.~G. Estrada, A note on designing logical circuits using {SAT}, in: {ICES},
  Springer Berlin Heidelberg, 2003, pp. 410--421.
\newblock \href {http://dx.doi.org/10.1007/3-540-36553-2_37}
  {\path{doi:10.1007/3-540-36553-2_37}}.

\bibitem{TrindadeSBESC2015}
A.~Trindade, H.~Ismail, L.~Cordeiro, Applying multi-core model checking to
  hardware-software partitioning in embedded systems, in: {SBESC}, 2015, pp.
  102--105.
\newblock \href {http://dx.doi.org/10.1109/SBESC.2015.26}
  {\path{doi:10.1109/SBESC.2015.26}}.

\bibitem{TrindadeSBESC2013}
A.~Trindade, L.~Cordeiro,
  \href{http://sbesc.lisha.ufsc.br/sbesc2014/dl185}{Aplicando verificação de
  modelos para o particionamento de hardware/software}, in: {SBESC}, 2014,
  p.~6.
\newline\urlprefix\url{http://sbesc.lisha.ufsc.br/sbesc2014/dl185}

\bibitem{TrindadeDAES2016}
A.~Trindade, L.~Cordeiro, Applying {SMT}-based verification to
  hardware/software partitioning in embedded systems, DES AUTOM EMBED SYST
  20~(1) (2016) 1--19.
\newblock \href {http://dx.doi.org/10.1007/s10617-015-9163-z}
  {\path{doi:10.1007/s10617-015-9163-z}}.

\bibitem{Cotton2011}
S.~Cotton, O.~Maler, J.~Legriel, S.~Saidi, Multi-criteria optimization for
  mapping programs to multi-processors, in: SIES, 2011, pp. 9--17.

\bibitem{CDC2016}
Y.~Shoukry, P.~Nuzzo, I.~Saha, A.~L. Sangiovanni{-}Vincentelli, S.~A. Seshia,
  G.~J. Pappas, P.~Tabuada, Scalable lazy {SMT}-based motion planning, in: 55th
  {IEEE} Conference on Decision and Control, {CDC} 2016, Las Vegas, NV, USA,
  December 12-14, 2016, 2016, pp. 6683--6688.
\newblock \href {http://dx.doi.org/10.1109/CDC.2016.7799298}
  {\path{doi:10.1109/CDC.2016.7799298}}.

\bibitem{ABsolver}
M.~Pister, M.~Tautschnig, A.~Bauer, Tool-support for the analysis of hybrid
  systems and models, 2007 10th Design, Automation and Test in Europe
  Conference and Exhibition 00 (2007) 172.
\newblock \href {http://dx.doi.org/10.1109/DATE.2007.364411}
  {\path{doi:10.1109/DATE.2007.364411}}.

\bibitem{Nuzzo2010}
P.~Nuzzo, A.~A.~A. Puggelli, S.~A. Seshia, A.~L. Sangiovanni-Vincentelli,
  Cal{CS}: {SMT} solving for non-linear convex constraints, Tech. Rep.
  {UCB}/{EECS}-2010-100, {EECS} Department, University of California, Berkeley
  (Jun 2010).

\bibitem{SMC}
Y.~Shoukry, P.~Nuzzo, A.~L. Sangiovanni-Vincentelli, S.~A. Seshia, G.~J.
  Pappas, P.~Tabuada, {SMC}: Satisfiability modulo convex optimization, in:
  Proceedings of the 20th ACM International Conference on Hybrid Systems:
  Computation and Control, HSCC '17, ACM, New York, NY, USA, 2017, p.~10.
\newblock \href {http://dx.doi.org/10.1145/3049797.3049819}
  {\path{doi:10.1145/3049797.3049819}}.

\bibitem{Bjorner2015}
N.~Bj{\o}rner, A.-D. Phan, L.~Fleckenstein, $\nu${Z} - an optimizing {SMT}
  solver, in: {TACAS}, Springer Berlin Heidelberg, 2015, pp. 194--199.
\newblock \href {http://dx.doi.org/10.1007/978-3-662-46681-0_14}
  {\path{doi:10.1007/978-3-662-46681-0_14}}.

\bibitem{Symba}
Y.~Li, A.~Albarghouthi, Z.~Kincaid, A.~Gurfinkel, M.~Chechik, Symbolic
  optimization with {SMT} solvers, in: Proceedings of the 41st ACM
  SIGPLAN-SIGACT Symposium on Principles of Programming Languages, POPL '14,
  ACM, New York, NY, USA, 2014, pp. 607--618.
\newblock \href {http://dx.doi.org/10.1145/2535838.2535857}
  {\path{doi:10.1145/2535838.2535857}}.

\bibitem{SebastianiCAV2015}
R.~Sebastiani, P.~Trentin, Opti{M}ath{SAT}: A tool for optimization modulo
  theories, in: {{CAV}}, Springer International Publishing, 2015, pp. 447--454.
\newblock \href {http://dx.doi.org/10.1007/978-3-319-21690-4_27}
  {\path{doi:10.1007/978-3-319-21690-4_27}}.

\bibitem{SebastianiTransactions2015}
R.~Sebastiani, S.~Tomasi, Optimization modulo theories with linear rational
  costs, {ACM TOCL} 16~(2) (2015) 12:1--12:43.
\newblock \href {http://dx.doi.org/10.1145/2699915}
  {\path{doi:10.1145/2699915}}.

\bibitem{SebastianiTACAS2015}
R.~Sebastiani, P.~Trentin, Pushing the envelope of optimization modulo theories
  with linear-arithmetic cost functions, in: {{TACAS}}, Springer Berlin
  Heidelberg, 2015, pp. 335--349.
\newblock \href {http://dx.doi.org/10.1007/978-3-662-46681-0_27}
  {\path{doi:10.1007/978-3-662-46681-0_27}}.

\bibitem{Pavlinovic0W15}
Z.~Pavlinovic, T.~King, T.~Wies, {Practical {SMT}-based type error
  localization}, in: {ICFP}, 2015, pp. 412--423.

\bibitem{Galperin19911}
E.~A. Galperin, Problem-method classification in optimization and control,
  Computers \& Mathematics with Applications 21~(6?7) (1991) 1 -- 6.
\newblock \href {http://dx.doi.org/10.1016/0898-1221(91)90155-W}
  {\path{doi:10.1016/0898-1221(91)90155-W}}.

\bibitem{refdet:floudas2000deterministic}
C.~Floudas, Deterministic Global Optimization, Nonconvex Optimization and Its
  Applications, Springer, 2000.

\bibitem{snyman2005practical}
J.~Snyman, Practical Mathematical Optimization: An Introduction to Basic
  Optimization Theory and Classical and New Gradient-Based Algorithms, Applied
  Optimization, Springer, 2005.

\bibitem{scholz2011deterministic}
D.~Scholz, Deterministic Global Optimization: Geometric Branch-and-bound
  Methods and their Applications, Springer Optimization and Its Applications,
  Springer New York, 2011.

\bibitem{FINDLER198741}
N.~V. Findler, C.~Lo, R.~Lo, Pattern search for optimization, Mathematics and
  Computers in Simulation 29~(1) (1987) 41 -- 50.
\newblock \href {http://dx.doi.org/10.1016/0378-4754(87)90065-6}
  {\path{doi:10.1016/0378-4754(87)90065-6}}.

\bibitem{refstc:marti2005stochastic}
K.~Marti, Stochastic Optimization Methods, Springer, 2005.

\bibitem{Gao2013}
S.~Gao, S.~Kong, E.~M. Clarke, dReal: An {SMT} Solver for Nonlinear Theories
  over the Reals, Springer Berlin Heidelberg, Berlin, Heidelberg, 2013, pp.
  208--214.
\newblock \href {http://dx.doi.org/10.1007/978-3-642-38574-2_14}
  {\path{doi:10.1007/978-3-642-38574-2_14}}.

\bibitem{Clarke1999}
E.~M. Clarke, Jr., O.~Grumberg, D.~A. Peled, Model Checking, MIT Press,
  Cambridge, MA, USA, 1999.

\bibitem{Baier2008}
C.~Baier, J.-P. Katoen, Principles of Model Checking (Representation and Mind
  Series), The MIT Press, 2008.

\bibitem{Beyer2016}
D.~Beyer, Reliable and Reproducible Competition Results with BenchExec and
  Witnesses (Report on SV-COMP 2016), Springer Berlin Heidelberg, 2016, pp.
  887--904.
\newblock \href {http://dx.doi.org/10.1007/978-3-662-49674-9_55}
  {\path{doi:10.1007/978-3-662-49674-9_55}}.

\bibitem{handbook09}
A.~Biere, {Bounded model checking}, in: Handbook of Satisfiability, IOS Press,
  2009, pp. 457--481.

\bibitem{BarrettSST09}
C.~W. Barrett, R.~Sebastiani, S.~A. Seshia, C.~Tinelli, {Satisfiability modulo
  theories}, in: Handbook of Satisfiability, IOS Press, 2009, pp. 825--885.

\bibitem{esbmc}
L.~Cordeiro, B.~Fischer, J.~Marques-Silva, {SMT}-based bounded model checking
  for embedded {ANSI-C} software, IEEE TSE 38~(4) (2012) 957--974.
\newblock \href {http://dx.doi.org/10.1109/TSE.2011.59}
  {\path{doi:10.1109/TSE.2011.59}}.

\bibitem{CBMC}
D.~Kroening, M.~Tautschnig, {CBMC} -- c bounded model checker (2014)
  389--391\href {http://dx.doi.org/10.1007/978-3-642-54862-8_26}
  {\path{doi:10.1007/978-3-642-54862-8_26}}.

\bibitem{MorseRCN014}
J.~Morse, M.~Ramalho, L.~Cordeiro, D.~Nicole, B.~Fischer, {ESBMC} 1.22 -
  (competition contribution), in: {TACAS}, 2014, pp. 405--407.

\bibitem{Abreu2016}
B.~R. Abreu, Y.~M.~R. Gadelha, C.~L. Cordeiro, B.~E. de~Lima~Filho, S.~W.
  da~Silva, Bounded model checking for fixed-point digital filters, JBCS 22~(1)
  (2016) 1--20.
\newblock \href {http://dx.doi.org/10.1186/s13173-016-0041-8}
  {\path{doi:10.1186/s13173-016-0041-8}}.

\bibitem{Bessa2016}
I.~V. Bessa, H.~I. Ismail, L.~C. Cordeiro, J.~E.~C. Filho, Verification of
  fixed-point digital controllers using direct and delta forms realizations,
  DES AUTOM EMBED SYST 20~(2) (2016) 95--126.
\newblock \href {http://dx.doi.org/10.1007/s10617-016-9173-5}
  {\path{doi:10.1007/s10617-016-9173-5}}.

\bibitem{esbmc_controller_SBESC}
I.~Bessa, H.~Ibrahim, L.~Cordeiro, J.~E. Filho, {Verification of Delta Form
  Realization in Fixed-Point Digital Controllers Using Bounded Model Checking},
  in: {SBESC}, 2014.

\bibitem{Beyer2011}
D.~Beyer, M.~E. Keremoglu, {CPA}checker: A tool for configurable software
  verification, in: CAV, Springer Berlin Heidelberg, 2011, pp. 184--190.
\newblock \href {http://dx.doi.org/10.1007/978-3-642-22110-1_16}
  {\path{doi:10.1007/978-3-642-22110-1_16}}.

\bibitem{Gadelha2015}
M.~Y.~R. Gadelha, H.~I. Ismail, L.~C. Cordeiro,
  \href{http://dx.doi.org/10.1007/s10009-015-0407-9}{Handling loops in bounded
  model checking of {C} programs via k-induction}, {STTT} 19~(1) (2017)
  97--114.
\newblock \href {http://dx.doi.org/10.1007/s10009-015-0407-9}
  {\path{doi:10.1007/s10009-015-0407-9}}.
\newline\urlprefix\url{http://dx.doi.org/10.1007/s10009-015-0407-9}

\bibitem{li2015selected}
L.~Li, Selected Applications of Convex Optimization, Springer Optimization and
  Its Applications, Springer Berlin Heidelberg, 2015.

\bibitem{Boyd:2004:CO:993483}
S.~Boyd, L.~Vandenberghe, Convex Optimization, Cambridge University Press, New
  York, NY, USA, 2004.

\bibitem{Pereira:2016:VCP:2851613.2851830}
P.~Pereira, H.~Albuquerque, H.~Marques, I.~Silva, C.~Carvalho, L.~Cordeiro,
  V.~Santos, R.~Ferreira,
  \href{http://doi.acm.org/10.1145/2851613.2851830}{Verifying {CUDA} programs
  using {SMT}-based context-bounded model checking}, in: Proceedings of the
  31st Annual {ACM} Symposium on Applied Computing, {SAC} '16, ACM, New York,
  NY, USA, 2016, pp. 1648--1653.
\newblock \href {http://dx.doi.org/10.1145/2851613.2851830}
  {\path{doi:10.1145/2851613.2851830}}.
\newline\urlprefix\url{http://doi.acm.org/10.1145/2851613.2851830}

\bibitem{DBLP:journals/corr/JamilY13}
M.~Jamil, X.~Yang, \href{http://arxiv.org/abs/1308.4008}{A literature survey of
  benchmark functions for global optimization problems}, CoRR abs/1308.4008.
\newline\urlprefix\url{http://arxiv.org/abs/1308.4008}

\bibitem{optoolboxmanual}
The Mathworks, Inc., Matlab Optimization Toolbox User's Guide (2016).

\bibitem{Rocha2015}
H.~Rocha, H.~Ismail, L.~Cordeiro, R.~Barreto, Model checking embedded c
  software using k-induction and invariants, in: Proceedings of VI Brazilian
  Symposium on Computing Systems Engineering, 2015, pp. 90--95.
\newblock \href {http://dx.doi.org/10.1109/SBESC.2015.24}
  {\path{doi:10.1109/SBESC.2015.24}}.

\end{thebibliography}


%
%
%
\end{document}